\documentclass{article}

\usepackage{microtype}
\usepackage{graphicx}
\usepackage{subfigure}
\usepackage{booktabs} 

\usepackage{hyperref}

\usepackage[accepted]{icml2023}

\usepackage{xfrac}
\usepackage[utf8]{inputenc} 
\usepackage[T1]{fontenc}    
\usepackage{hyperref}       
\usepackage{url}            
\usepackage{booktabs}       
\usepackage{nicefrac}       
\usepackage{microtype}      
\usepackage{xcolor}         
\usepackage{hyperref}
\usepackage{amsfonts}      
\usepackage{amsmath}
\usepackage{amssymb}
\usepackage{amsthm}
\usepackage{babel}
\usepackage{adjustbox}
\usepackage{bbm}
\usepackage{wrapfig}
\usepackage{graphicx,subfigure}
\usepackage{caption}
\usepackage[makeroom]{cancel}

\newenvironment{talign*}
 {\csname align*\endcsname}
 {\endalign}
\newenvironment{talign}
{\align}
{\endalign}

\usepackage{listings}
\usepackage{xparse}
\usepackage{xcolor}
\usepackage{epstopdf}
\usepackage{epsfig}
\usepackage{enumitem}
\usepackage{dblfloatfix}

\usepackage[capitalize,noabbrev]{cleveref}

\theoremstyle{plain}
\newtheorem{theorem}{Theorem}[section]
\newtheorem{proposition}[theorem]{Proposition}

\theoremstyle{definition}
\newtheorem{definition}[theorem]{Definition}

\theoremstyle{remark}

\setlist[itemize]{leftmargin=5mm,itemsep=0.5mm}
\setlist[enumerate]{leftmargin=*,itemsep=0.5mm}
\usepackage[textsize=tiny]{todonotes}

\icmltitlerunning{Markovian Gaussian Process Variational Autoencoders}

\usepackage{mathtools}

\newcommand{\indep}{\perp \!\!\! \perp}
\newcommand{\bA}{\mathbf{A}}
\newcommand{\ba}{\mathbf{a}}
\newcommand{\bQ}{\mathbf{Q}}
\newcommand{\bm}{\mathbf{m}}
\newcommand{\bP}{\mathbf{P}}
\newcommand{\bs}{\mathbf{s}}
\newcommand{\bY}{\mathbf{Y}}

\newcommand{\bZ}{\mathbf{Z}}

\newcommand{\bK}{\mathbf{K}}

\newcommand{\bz}{\mathbf{z}}
\newcommand{\bH}{\mathbf{H}}
\newcommand{\bF}{\mathbf{F}}
\newcommand{\bG}{\mathbf{G}}
\newcommand{\bV}{\mathbf{V}}

\begin{document}

\twocolumn[
\icmltitle{Markovian Gaussian Process Variational Autoencoders}



\icmlsetsymbol{equal}{*}

\begin{icmlauthorlist}
\icmlauthor{Harrison Zhu}{yyy,equal}
\icmlauthor{Carles Balsells-Rodas}{yyy,equal}
\icmlauthor{Yingzhen Li}{yyy}
\end{icmlauthorlist}

\icmlaffiliation{yyy}{Imperial College London}

\icmlcorrespondingauthor{Harrison Zhu}{harrisonzhu5080@gmail.com or hbz15@ic.ac.uk}

\icmlkeywords{Machine Learning, ICML}

\vskip 0.3in
]



\printAffiliationsAndNotice{\icmlEqualContribution} 

\begin{abstract}
Sequential VAEs have been successfully considered for many high-dimensional time series modelling problems, with many variant models relying on discrete-time mechanisms such as recurrent neural networks (RNNs). On the other hand, continuous-time methods have recently gained attraction, especially in the context of irregularly-sampled time series, where they can better handle the data than discrete-time methods. One such class are Gaussian process variational autoencoders (GPVAEs), where the VAE prior is set as a Gaussian process (GP). However, a major limitation of GPVAEs is that it inherits the cubic computational cost as GPs, making it unattractive to practioners. In this work, we leverage the equivalent discrete state space representation of Markovian GPs to enable linear time GPVAE training via Kalman filtering and smoothing. For our model, Markovian GPVAE (MGPVAE), we show on a variety of high-dimensional temporal and spatiotemporal tasks that our method performs favourably compared to existing approaches whilst being computationally highly scalable.
\end{abstract}

\section{Introduction}

\begin{figure}[t]
\centering
\vspace{-2mm}
\begin{minipage}{0.37\textwidth}
\centering
\includegraphics[width = \columnwidth]{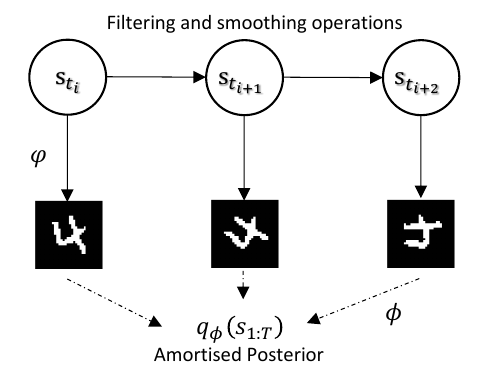}
\vspace{-4mm}

\end{minipage}
\hfill
\begin{minipage}{0.37\textwidth}
\centering
\includegraphics[width = \columnwidth]{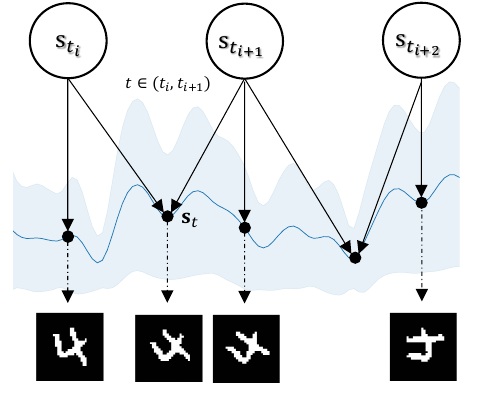}
\end{minipage}
\vspace{-2mm}
\caption{Illustration of the MGPVAE model. (Top) The state posterior $q_\phi(\bs_{1:T})$ is parameterised by encoder outputs and computed using filtering and smoothing. (Bottom) At prediction time, posterior predictive distributions can be calculated at any $t$.}
\label{fig:gpvae}
\vspace{-5mm}
\end{figure}

Modelling multivariate time series data has extensive applications in e.g., video and audio generation \citep{li2018disentangled, goel2022s}, climate data analysis \citep{ravuri2021skilful} and finance \citep{sims1980macroeconomics}. Among existing deep generative models for time series, a popular class of model is sequential variational auto-encoders (VAEs) \citep{chung2015recurrent,fraccaro2017disentangled,fortuin_gp-vae_2020}, which extend VAEs \citep{kingma_auto-encoding_2014} to sequential data. Originally proposed for image generation, a VAE is a latent variable model which encodes data via an \textbf{encoder} to a low-dimensional latent space and then decodes via a \textbf{decoder} to reconstruct the original data. To extend VAEs to sequential data, the latent space must also include temporal information (it is also technically possible to place temporal dynamics on the decoders \citep{chen2017variational}, but for our work we focus on the latent variable dynamics). Sequential VAEs accomplish this by modelling the latent variables as a multivariate time series, where many existing approaches define a state-space model which governs the latent dynamics. These state-space model-based sequential VAE approaches can be classified into two subgroups:  

\vspace{-3mm}
\begin{itemize}
\item Discrete-time:
The first approach relies on building a \textbf{discrete-time state-space model} for the latent variables. The transition distribution is often parameterised by a recurrent neural network (RNN) such as Long Short-Term Memory (LSTM; \citet{hochreiter1997long}) or Gated Recurrent Units (GRU; \citet{cho-etal-2014-properties}). Notable methods include the Variational Recurrent Neural Network (VRNN; \citet{chung2015recurrent}) and Kalman VAE (KVAE; \citet{fraccaro2017disentangled}). However, these approaches may suffer from training issues, such as vanishing gradients \citep{pascanu2013difficulty}, and may struggle with irregularly-sampled time series data \citep{rubanova2019latent}. 

\item Continuous-time:
The second approach involves \textbf{continuous-time} representations, where the latent space is modelled using a continuous-time dynamic model. A notable class of such methods are neural differential equations \citep{chen2018neural, rubanova2019latent, li2020scalable, kidger2022neural}, which model the latent variables using a system of differential equations, described by its initial conditions, drift and diffusion. As remarked in \citet{li2020scalable}, one can construct a neural stochastic differential equation (neuralSDE) that can be interpreted as an infinite-dimensional noise VAE. However, although these models can flexibly handle irregularly-sampled data, they also require numerical solvers to solve the underlying latent processes, which may cause training difficulties \citep{park2020vid}, memory issues \citep{chen2018neural, li2020scalable} and slow computation times. Similarly, but combining linear SDEs with Kalman filtering, Continuous Recurrent Units (CRU; \citet{schirmer2022modeling}) is a RNN that is also able to model continuous data. Finally, in the context of audio generation, S4-related models \citep{gu2021efficiently, goel2022s} also rely on continuous state spaces and have been shown to perform strongly.
\end{itemize}
\vspace{-2mm}

Another line of continuous-time approaches, related to neuralSDEs, that treat the latent multivariate time series as a random function of time, and model the random function as a tractable stochastic process, are Gaussian Process Variational Autoencoders (GPVAEs; \citet{casale_gaussian_2018,pearce_gaussian_2020,fortuin_gp-vae_2020,ashman_sparse_2020, jazbec_scalable_2021}) which model the latent variables using Gaussian processes (GPs) \citep{rasmussen2003gaussian}. As compared with dynamic model-based approaches which focus on modelling the latent variable transitions (reflecting local properties mainly), the GP model for the latent variables describes, in a better way, the global properties of the time series if a suitable stationary kernel is chosen, such as smoothness and periodicity. Therefore GPVAEs may be better suited for e.g., climate time series data which clearly exhibits periodic behaviour. Unfortunately GPVAEs are not directly applicable to long sequences as they suffer from $\mathcal{O}(T^3)$ computational cost, therefore approximations need to be made. Indeed, \citet{ashman_sparse_2020, fortuin_gp-vae_2020, jazbec_scalable_2021} proposed variational approximations based on sparse Gaussian processes \citep{titsias2009variational,hensman2013gaussian,hensman2015scalable} or recognition networks \citep{fortuin_gp-vae_2020} to improve the scalability of GPVAEs.

In this work we propose Markovian GPVAEs (MGPVAEs) to bridge state-space model-based and stochastic process based approaches of sequential VAEs, aiming to achieve the best in both worlds. Our approach is inspired by the key fact that, when the GP is over time, a large class of GPs can be written as a linear SDE \citep{sarkka2019applied}, for which there exists exact and unique solutions \citep{oksendal_stochastic_2003}. As a result, there exists an equivalent discrete linear state space representation of GPs. Therefore the dynamic model for the latent variables has both discrete and continuous-time representations. This brings the following key advantages to the latent dynamic model of MGPVAE: 

\vspace{-3mm}
\begin{itemize}
    \item The continuous-time representation allows the incorporation of inductive biases via the GP kernel design (e.g., smoothness, periodic and monotonic trends), to achieve better prediction results and training efficiency. It also enables modelling irregularly sampled time series data. 
    \item The equivalent discrete-time representation, which is linear, enables Kalman filtering and smoothing \citep{sarkka2019applied,adam_doubly_2020, chang_fast_2020,wilkinson_state_2020,wilkinson_sparse_2021, hamelijnck_spatio-temporal_2021} that computes the posterior distributions in $\mathcal{O}(T)$ time. As the observed data is assumed to come from non-linear transformations of the latent variables, we further apply site-based approximations \citep{chang_fast_2020} for the non-linear likelihood terms to enable analytic solutions for the filtering and smoothing procedures.
\end{itemize}
\vspace{-3mm}

In our experiments, We study much longer datasets ($T\approx 100)$ compared to many previous GPVAE and discrete-time works, which are only are of the magnitude of $T\approx 10$. We include a range of datasets that describe different properties of MGPVAE compared to existing approaches:

\vspace{-3mm}
\begin{itemize}
    \item We deliver competitive performance compared to many existing methods on corrupt and irregularly-sampled video and robot action data at a fraction of the cost of many existing models. 
    \item We extend our work to spatiotemporal climate data, where none of the discrete-time sequential VAEs are suited for modelling. We show that it outperforms traditional GP and existing sparse GPVAE models in terms of both predictive performance and speed.
\end{itemize}

\section{Background}
Consider building generative models for high-dimensional time series (e.g., video data). Here an observed sequence of length $T$ is denoted as $\bY_{t_1}\ldots,\bY_{t_T}\in\mathbb{R}^{D_y}$, where $t_i$ represents the timestamp of the $i$th observation in the sequence. Note that in general $t_i \neq i$ for irregularly sampled time series. As the proposed MGPVAE has both discrete state-space model based and stochastic process based formulations, below we introduce these two types of sequential VAEs and the key relevant techniques.

\paragraph{Sequential VAEs with state-space models:} Consider $t_i = i$ w.l.o.g., and assume each of the latent states in $\bZ_{1:T}=(\bz_{1:T}^1,\ldots, \bz_{1:T}^L)\in\mathbb{R}^{T\times L}$ has $L$ latent dimensions. Then the generative model is defined as
\begin{align}
    p(\bY_{1:T}, \bZ_{1:T}) = p(\bZ_{1:T}) \prod_{t=1}^T p(\bY_t|\bZ_t),
\label{eq:generative_model_general}
\end{align}
where we choose $p(\bY_t | \bZ_t) = \mathcal{N}(\bY_t; \varphi(\bZ_t), \sigma^2 \mathbf{I})$ to be a multivariate Gaussian distribution, and $\varphi:\mathbb{R}^L\rightarrow\mathbb{R}^{D_y}$ is decoder network that transforms the latent state to the Gaussian mean. The prior $p(\bZ_{1:T})$ is defined by the transition probabilities, e.g., $p(\bZ_{1:T}) = \prod_{t=1}^T p(\bZ_t | \bZ_{<t})$. Training is done by maximising the variational lower bound $\mathcal{L}$ \citep{ranganath2014black}:
\begin{talign}
\label{eqn:ELBO}
 \log p(\bY_{1:T})\geq& \sum_{t=1}^T \mathbb{E}_{q(\bZ_{1:T})}[\log p(\bY_t | \bZ_t)] \nonumber\\
 &- \text{KL}(q(\bZ_{1:T} | \bY_{1:T}) || p(\bZ_{1:T})):=\mathcal{L},
\end{talign}
where $q(\bZ_{1:T} | \bY_{1:T})$ is the approximate posterior parameterised by an encoder network. Often this $q$ distribution is defined by mean-field approximation over time, i.e.,~$q(\bZ_{1:T} | \bY_{1:T}) = \prod_{t=1}^T q(\bZ_t | \bY_t)$, or by using transition probabilities as well, e.g., $q(\bZ_{1:T} | \bY_{1:T}) = \prod_{t=1}^T q(\bZ_t | \bZ_{t-1}, \bY_{\leq t})$. Below we also write $q(\bZ_{1:T}) = q(\bZ_{1:T} | \bY_{1:T})$ to simplify notation.

\paragraph{Gaussian Process Variational Autoencoders:} 
A Gaussian process is a stochastic process denoted by $\bZ\sim\mathcal{GP}(0,k)$, where $k:\mathbb{R}\times\mathbb{R}\rightarrow\mathbb{R}$ is the kernel or covariance function that specifies the similarity between two time stamps. This allows us to explicitly enforce inductive biases or global behaviour. For instance, if $\bY_t$ is a periodic system, then we may use a periodic kernel $k$ or a longer initial kernel lengthscale to incorporate this knowledge in the model; if the underlying process is smooth, then a kernel can also be chosen so that it induces smooth functions \citep{kanagawa_gaussian_2018}. 

GPVAEs \citep{casale_gaussian_2018,pearce_gaussian_2020,fortuin_gp-vae_2020,ashman_sparse_2020, jazbec_scalable_2021} define the decoder network $\varphi$ and the conditional distribution $p(\bY_t | \bZ_t)$ in the same way as presented above. However, instead of using transition probabilities, a GPVAE places a multi-output GP prior on the latent variables $\{ \bZ_t \}$: $\bZ_t\sim \mathcal{GP}(0, \mathbf{k})$, where one can choose the kernel of the multi-output GP to be $\mathbf{k}=\mathbf{I}\otimes k$, i.e., the output GPs across dimensions share the same kernel $k$. However, each dimension may also be induced with separate kernels $k^1,\ldots,k^L$, which we adopt in this work, giving block diagonal kernel matrices.

Again we use the variational lower-bound (Eq.~\eqref{eqn:ELBO} when $t_i = i$) as the training objective, but with a different approximate posterior $q$. Some examples include GPVAE \citep{pearce_gaussian_2020, fortuin_gp-vae_2020}
\begin{talign*}
q(\bZ_{1:T}) = \prod_{l=1}^L N(\bZ^l_{1:T}; \tilde{\bY}_{1:T}^l, \tilde{\bV}_{1:T}^l),
\end{talign*}
where $(\tilde{\bY}^l_t,\tilde{\bV}^l_t)_{l=1}^L =(\mu_\phi^l(\bY_t), \mathbf{\Sigma}^l_\phi(\bY_t))_{l=1}^L=\phi(\bY_t)$ are outputs of the encoder network $\phi$. This corresponds to a mean-field approximation over latent dimensions instead of time. To avoid direct parameterisation of the full covariance matrix $\mathbf{\Sigma}^l_\phi(\bY_{1:T})\in\mathbb{R}^{T\times T}$ which can be expensive for long sequences, \citet{fortuin_gp-vae_2020} proposed a banded parameterisation of the precision matrix \citep{blei2006dynamic,bamler2017dynamic}, reducing both the time and memory complexity to $\mathcal{O}(T)$. However, this choice makes it more difficult to work with irregularly-sampled data and previous works only focused on corrupt video frames. 

Another more flexible option is to use sparse GP approximations with inducing points \citep{jazbec_scalable_2021}:

\vspace{-5mm}
\begin{talign*}
q(\bZ_{1:T}) &= \prod_{l=1}^L p(\bZ^l_{1:T}|\bG^l_m)q(\bG^l_m),\\
q(\bG^l_m) &= N(\bG^l_m | \mathbf{m}^l_m, \mathbf{A}_m^l),\\
\mathbf{S}^l&=\mathbf{K}_{mm}^l + \mathbf{K}_{mT}^l \text{diag}(\tilde{\bV}^l_{1:T})^{-1}\mathbf{K}_{Tm}^l,
\\
\mathbf{m}^l_m&= \mathbf{K}_{mm}^l (\mathbf{S}^l)^{-1}\mathbf{K}_{mT}^l\text{diag}(\tilde{\bV}^l_{1:T})^{-1}\tilde{\bY}^l_{1:T},\\
\mathbf{A}^l_m&=\mathbf{K}_{mm}^l(\mathbf{S}^l)^{-1}\mathbf{K}_{mm}^l,
\end{talign*}
\vspace{-5mm}

where $p(\bZ^l_{1:T}|\mathbf{G}_m^l)$ is the standard multivariate Gaussian conditional distribution
and $[\mathbf{K}_{mT}^l]_{ij}:=k^l(\mathbf{U}_i, j)$ with $i=1,\ldots,m$ and $j=1,\ldots,T$, and $[\mathbf{K}_{mm}^l]_{ij}:=k^l(\mathbf{U}_i, \mathbf{U}_j)$ with $i,j=1,\ldots,m$, for pre-determined inducing time locations $\mathbf{U}=[\mathbf{U}_1,\ldots,\mathbf{U}_m]^\intercal$. However, the time complexity scales with $\mathcal{O}(m^3 + m^2T)$, where in practice to attain good performance $m=\mathcal{O}(\log T)$ \citep{burt2019rates}, and therefore the complexity increases massively when dealing with longer time series.

\paragraph{Markovian Gaussian Processes:} Interestingly, the banded parameterisation coincides with the structure of the Markovian GP state space $\mathbf{s}_t$. w.l.o.g. suppose $\bZ_t$ is one-dimensional in this subsection, then with a conjugate likelihood (e.g.~linear Gaussian) $p(\bY_t|\bZ_t)$ and a Markovian kernel $k$ \citep{sarkka2019applied}, we can write the GP regression problem as an Itô SDE of latent dimension $d$ 
\begin{talign}
  \label{eqn:temporal_sde}
  &\text{d}\mathbf{s}_t = \mathbf{F}\mathbf{s}_t \text{d}t + \mathbf{L}\text{d}B_{t},\quad \bZ_t=\mathbf{H}\mathbf{s}_t,\nonumber\\
  & \bY_t|\bZ_t\sim p(\bY_t| \bZ_t),
\end{talign}
where $\mathbf{F}\in\mathbb{R}^{d\times d}, \mathbf{L}\in\mathbb{R}^{d\times e}, \mathbf{H}\in\mathbb{R}^{1\times d}$ are the feedback, noise effect and emission matrices, respectively, and $B_t$ is an $e$-dimensional (correlated) Brownian motion with diffusion $\mathbf{Q}_c$. $\bs_0\sim \mathcal{N}(0, \bP_\infty)$, where $\bP_\infty$ is the stationary state covariance, which satisfies the Lyapunov equation \citep{solin_stochastic_2016}. The state is typically the $d$ derivatives $\mathbf{s}_t=(\bZ_t,\bZ^{(1)}_t,\ldots,\bZ_t^{(d-1)})^\intercal$ with the subsequent emission matrix $\mathbf{H}=(1,0,\ldots,0)$. 

The linear SDE in Eq.~\eqref{eqn:temporal_sde} admits a unique closed form solution, allowing the recursive updates of $\mathbf{s}_{t_{i+1}}$ given $\mathbf{s}_{t_i}$:
\begin{talign}
\label{eqn: markov_gp}
  &\mathbf{s}_{t_{i+1}} = \mathbf{A}_{i,i+1}\mathbf{s}_{t_i} + \mathbf{q}_i,\quad \mathbf{q}_i\sim\mathcal{N}(0, \mathbf{Q}_{i,i+1}), \nonumber \\ 
  &\bZ_{t_i} = \mathbf{H}\mathbf{s}_{t_i},\quad \bY_{t_i}|\bZ_{t_i}\sim p(\bY_{t_i}| \bZ_{t_i})
\end{talign}
\begin{talign*}
  &\text{with } \mathbf{A}_{i,i+1}=e^{\Delta_i\mathbf{F}},\quad \text{where }\Delta_i=t_{i+1}-t_i,\\
  &\mathbf{Q}_{i,i+1}=\int_{t_0}^{\Delta_i+t_0}e^{(\Delta_i+t_0-\tau)\mathbf{F}}\mathbf{L}\mathbf{Q}_c\mathbf{L}^\intercal  [e^{(\Delta_i+t_0-\tau)\mathbf{F}}]^\intercal \text{d}\tau.
\end{talign*}
Note that the $\mathbf{Q}_{i,i+1}$ can be easily obtained in closed form. See Appendix~\ref{appendix:mgpvae} for a detailed explanation. For conjugate likelihood $p(\bY_{t}| \bZ_{t})\equiv p(\bY_{t}|\bs_{t})$, we can use the recursive Kalman filtering and smoothing equations (also known as the forward-backward equations) to obtain the posterior distribution \citep{sarkka2019applied} $p(\bs_{t}|\bY_{1:T})$ with $\mathcal{O}(Td^3)$ complexity. The corresponding filter prediction, filtering and smoothing equations are:

\vspace{-6mm}
\begin{talign} 
\label{eqn:filtering_smoothing}
p(\bs_{t+1}|\bY_{1:t}) &= \int p(\bs_{t+1}|\bs_t)p(\bs_t|\bY_{1:t}) \text{d}\bs_t, \\
p(\bs_t|\bY_{1:t}) &= \frac{1}{\ell_t}p(\bY_t|\bs_t)p(\bs_t|\bY_{1:t-1}),\nonumber\\
    p(\bs_{t}|\bY_{1:T})&=p(\bs_{t}|\bY_{1:t})\int \frac{p(\bs_{t+1}|\bs_t)p(\bs_{t+1}|\bY_{1:T})}{p(\bs_{t+1}|\bY_{1:t})} \text{d}\bs_{t+1}, \nonumber
\end{talign}
\vspace{-6mm}

where $\ell_t=\int p(\bY_t|\bs_t)p(\bs_t|\bY_{1:t-1}) \text{d}\bs_t$ is tractable as the integrand is a product of Gaussians. If $p(\bY_{t}|\bs_{t})$ is non-conjugate (e.g. a Poisson distribution), then the posterior cannot be obtained analytically.

The size of $d$ depends on the kernel. For example, the Matern-$\sfrac{3}{2}$ kernel yields $d=2$, since the GP sample paths lie in an RKHS that is norm equivalent to a space of functions with 1 derivative \citep{kanagawa_gaussian_2018}. The periodic or quasi-periodic kernels \citep{Solin2014ExplicitLB} may yield larger $d$'s. However, in comparison to sparse Gaussian process approximations in \citet{ashman_sparse_2020, jazbec_scalable_2021} that have complexity $\mathcal{O}(m^3+m^2T)$, where $m$ is the number of inducing points, $d$ does not depend on $T$ (unlike sparse GPs \citep{burt2019rates} that depend on $\mathcal{O}(\log^D T)$, where $D$ is the time variable dimension) and thus does not need to grow as $T$ increases.

\section{Markovian Gaussian Process Variational Autoencoders}
In this section, we propose Markovian GPVAEs (MGPVAEs) and a corresponding variational inference scheme for model learning.

\subsection{Model} 
Let $L$ be the dimensionality of the latent variables and $k^l$ be the kernel for the $l$th channel with state dimension $d_l$. Then we have a total state space dimension of $\sum_{l=1}^L d_l$. Let $\varphi:\mathbb{R}^L\rightarrow \mathbb{R}^{D_y}$ be the decoder network. Then the generative model, under the linear SDE form of Markovian GP in the latent space, is (with $\bZ_t=(\bZ_t^1,\ldots,\bZ_t^L)^\intercal$)

\vspace{-6mm}
\begin{talign}
  \label{eqn:mgpvae}
  &\text{d}\mathbf{s}^l_t = \mathbf{F}^l\mathbf{s}^l_t \text{d}t + \mathbf{L}^l\text{d}B_{t}^l,\quad \bZ_t^l=\mathbf{H}^l\mathbf{s}^l_t,\quad l=1,\ldots, L\nonumber\\
  & \bY_t|\bZ_t\sim p(\bY_t| \bZ_t)\equiv p(\bY_t| \varphi(\bZ_t)), 
\end{talign}
\vspace{-6mm}

which equivalently becomes the linear discrete state space model with nonlinear likelihood
\begin{talign}
  &\mathbf{s}_{t_{i+1}}^l = \mathbf{A}^l_{i,i+1}\mathbf{s}^l_{t_i} + \mathbf{q}^l_i,\quad \mathbf{q}^l_i\sim\mathcal{N}(0, \mathbf{Q}^l_{i,i+1}),\nonumber\\
  & \bZ_t^l=\mathbf{H}^l\mathbf{s}^l_t, \quad \bY_t|\bZ_t\sim p(\bY_t| \bZ_t)\equiv p(\bY_t| \varphi(\bZ_t)).
\end{talign}
Note that the transformation $\bZ_t^l=\bH^l\bs_t^l$ is deterministic, and the stochasticity arises from the $\bs_t$ variables.

\subsection{Variational inference} 
Suppose $q(\mathbf{s})$ is the approximate posterior over $\bs:=\bs_{1:T}\in\mathbb{R}^{T\times Ld}$. Then minimising $\text{KL}(q(\mathbf{s})||p(\mathbf{s}|\bY))$ is equivalent to maximising the lower bound $\mathcal{L}:=\sum_{t=1}^T\mathbb{E}_{q(\bs_t)}\log p(\bY_t|\varphi(\bZ_t)) - \text{KL}(q(\mathbf{s})||p(\mathbf{s}))$.
Note that since $\bZ_t$ is a linear transformation or reparameterisation of $\bs_t$, the KL-divergence is between the posterior and prior distributions of $\bs_t$. We wish to compute $q(\bs)$ in linear time using Kalman filtering and smoothing. However, due to the presence of a nonlinear decoder network $\varphi$ in the likelihood, it is no longer possible to obtain the exact posterior $p(\bs|\bY)$ due to non-conjugacy. However, if we approximate the likelihood $p(\bY_t|\varphi(\bZ_t))$ with Gaussian sites, as is done in \citet{pearce_gaussian_2020,ashman_sparse_2020, jazbec_scalable_2021, chang_fast_2020}, Kalman filtering and smoothing can be performed as conjugacy is reintroduced in the filtering and smoothing equations.

We propose the Gaussian-site approximation
\begin{talign*}
q(\bs)\propto p(\bs) \prod_{l=1}^L\prod_{t=1}^T N(\tilde{\bY}_t^l| \mathbf{H}^l\bs^l_t, \tilde{\bV}^l_t),    
\end{talign*}
where $\tilde{\bY}^l_t\in\mathbb{R}$ and $\tilde{\bV}^l_t\in\mathbb{R}$ for $l=1,\ldots,L$. Instead of optimising $\tilde{\bY}_t$ and $\tilde{\bV}_t$ as free-form parameters using conjugate-computation variational inference \citep{khan2017conjugate}, we encode them using outputs of an encoder network $\phi$ i.e. $(\tilde{\bY}^l_t,\tilde{\bV}^l_t)_{l=1}^L=\phi(\bY_t)$. In addition, we approximate the potentially high-dimensional data likelihood using a likelihood comprising of low-dimensional state space variables.

Then with straightforward computations, 
\begin{talign*}
    &\text{KL}(q(\mathbf{s})||p(\mathbf{s})) = \mathbb{E}_{q(\bs)} \log \frac{q(\bs)}{p(\bs)} \\
    &= \mathbb{E}_{q(\bs)} \log \frac{p(\bs)\prod_{l=1}^L\prod_{t=1}^T N(\tilde{\bY}_t^l| \mathbf{H}^l\bs^l_t, \tilde{\bV}^l_t)}{p(\bs) \int p(\bs)\prod_{l=1}^L\prod_{t=1}^T N(\tilde{\bY}_t^l| \mathbf{H}^l\bs^l_t, \tilde{\bV}^l_t) \text{d}\bs}\\
    &= \mathbb{E}_{q(\bs)}\sum_{l=1}^l\sum_{t=1}^T \log N(\tilde{\bY}_t^l| \mathbf{H}^l\bs^l_t, \tilde{\bV}^l_t)\\
    &- \log \mathbb{E}_{p(\bs)} \prod_{t=1}^T\prod_{l=1}^L N(\tilde{\bY}_t^l| \mathbf{H}^l\bs^l_t, \tilde{\bV}^l_t).
\end{talign*}
The ELBO (\ref{eqn:ELBO}) thus becomes
\begin{align*}
  &\mathcal{L}=
  \underbrace{\log\mathbb{E}_{p(\bs)} \left[\prod_{t=1}^T \prod_{l=1}^L N(\tilde{\bY}^l_t| \mathbf{H}^l\bs^l_t, \tilde{\bV}^l_t) \right]}_{\text{E3}}\\
  &+\sum_{t=1}^T \mathbb{E}_{q(\bs_t)}\bigg[ \underbrace{\log p(\bY_t|\varphi(\bZ_t))}_{\text{E2}} -\sum_{l=1}^L\underbrace{\log N(\tilde{\bY}^l_t| \mathbf{H}^l\bs^l_t, \tilde{\bV}^l_t)}_{\text{E1}}\bigg].
\end{align*}
(E1) $q(\bs_t)$ can be obtained using the Kalman smoothing distributions and can be computed in linear time (see the Kalman filtering and smoothing equations in Appendix~\ref{alg:filtering_smoothing}). The reason is because $q(\bs)$ is constructed by replacing the likelihood in Eq.~(\ref{eqn: markov_gp}) with Gaussian sites, making it possible to analytically evaluate the filtering and smoothing equations (\ref{eqn:filtering_smoothing}). Therefore using the same calculations as \citet{chang_fast_2020, hamelijnck_spatio-temporal_2021}, the first term in the ELBO can be computed analytically:
\begin{talign*}
  &\mathbb{E}_{q(\bs_t)} [\text{E1}] =- \frac{1}{2}\log |2\pi \tilde{\bV}^l_t| - \frac{1}{2}(\tilde{\bY}^l_t)^\intercal (\tilde{\bV}^l_t)^{-1}  \tilde{\bY}_t^l\\ & + (\tilde{\bY}^l_t)^\intercal \bH^l\bm_t^{l,s} - \frac{1}{2}[\text{Tr}((\tilde{\mathbf{V}}^l)^{-1} \bH^l \bP_t^{l,s}(\bH^l)^\intercal)] \\
  &+ (\bm_t^{l,s})^\intercal (\bH^l)^\intercal \bH^l \bm_t^{l,s},
\end{talign*}
where $\bm_t^{l,s}$ and $\bP_t^{l,s}$ are the smoothing mean and covariances respectively at time $t$.

(E2) is intractable but we can estimate it using Monte Carlo with $K$ samples $\bs_{t, j}\sim q(\bs_t)$, and thus samples $\bZ_{t,j}=(\bH^1\bs_{t, j}^1,\ldots,\bH^L\bs_{t, j}^L)^\intercal$:
\begin{talign*}
  \mathbb{E}_{q(\bs_t)}\log p(\bY_t|\varphi(\bZ_t)) \approx \frac{1}{K}\sum_{j=1}^K \log p(\bY_t| \varphi(\bZ_{t,j})).
\end{talign*}

(E3) is the log partition function of $q(\bs)$, which is also the log marginal likelihood of the approximate model $\sum_{l=1}^L\log p(\tilde{\bY}^l)$. Note that the latent channels are independent of each other, allowing us to sum over the log marginal likelihood over each channel.
We can further decompose each term of the sum into
\begin{talign*}
\log p(\tilde{\bY}^l) &= \log p(\tilde{\bY}_1^l)\prod_{t=2}^T p(\tilde{\bY}^l_t|\tilde{\bY}^l_{1:t-1})\\
&=\sum_{t=1}^T \log\mathbb{E}_{p(\bs_t^l|\tilde{\bY}^l_{1:t-1})} N(\tilde{\bY}^l_t; \mathbf{H}^l\bs^l_t, \tilde{\bV}^l_t),
\end{talign*}
where $p(\bs_t^l|\tilde{\bY}^l_{1:t-1})$ is the predictive filter distribution obtained with Kalman filtering. Fortunately, $\log p(\tilde{\bY}^l)$ can be computed during the filtering stage (see Algorithm~\ref{alg:filtering_smoothing} for a full breakdown of Kalman filtering and smoothing). In addition, a graphical representation of the Markovian GPVAE is in Figure~\ref{fig:gpvae}.

\subsection{Spatiotemporal Modelling} 
Spatiotemporal modelling is an important task with many real world applications \citep{cressie2015statistics}. Traditional methods such as kriging, or Gaussian process regression, incurs cubic computational costs and are even more costly and difficult when multiple variables need to be modelled jointly. GPVAEs may ameliorate this issue by effectively simplifying the task via an encoder-decoder model and has been proven to be effective in \citet{ashman_sparse_2020}. Following section 4.2 of \citet{hamelijnck_spatio-temporal_2021}, given a separable spatiotemporal kernel $k(r,t,r',t')=k_r(r,r')k_t(t,t')$, it is straightforward to extend MGPVAE to model spatiotemporal data, which will be make it a highly scalable spatiotemporal model. We consider the model:
\begin{talign*}
    \bZ(r,t)\sim \mathcal{GP}(0,k),\bY(r,t)|\bZ(r,t)\sim p(\bY(r,t)| \varphi(\bZ(r,t))),
\end{talign*}
where for classical kriging \citep{cressie2015statistics}, $\varphi$ is the identity map and $\bZ(r,t)$ is of the same dimensionality as $\bY(r,t)$.

For convenience of notation, we avoid introducing subscripts $l$ and only write down 1 latent dimension with $k$. Suppose we have $N_s$ spatial coordinates observed over time, denoted by the spatial matrix $\mathbf{R}\in\mathbb{R}^{N_r\times D_x}$, it is possible to rewrite the GP regression model by stacking the states for each spatial location on top of each other to get:
\begin{talign}
  &\mathbf{s}_{t_{i+1}} = \mathbf{A}_{i,i+1}\mathbf{s}_{t_i} + \mathbf{q}_i,\quad \mathbf{q}_i\sim\mathcal{N}(0, \mathbf{Q}_{i,i+1}), \nonumber \\ 
  &\bY_{t}|\bZ_{t}\sim p(\bY_{t}| \varphi(\bZ_{t})),
\end{talign}
where $\bs_t=[\bs_t(r_1),\ldots,\bs_t(r_{N_s})]^\intercal$, $\bZ_t=[\mathbf{L}_{\mathbf{R}\mathbf{R}}^{r}\otimes \bH^{t}] \bs_{t_i}$ and $\mathbf{A}_{i,i+1}=\mathbf{I}_{N_r}\otimes \bA_{i,i+1}^{t}$, $\bQ_{i,i+1}=\mathbf{I}_{N_r}\otimes \bQ_{i,i+1}^{t}$,  $\mathbf{K}^r_{\mathbf{R}\mathbf{R}}=\mathbf{L}_{\mathbf{R}\mathbf{R}}^{r}(\mathbf{L}_{\mathbf{R}\mathbf{R}}^{r})^\intercal$, where superscripts $t$ and $r$ indicate the temporal state space and spatial kernel matrices respectively. The graphical model for MGPVAE is shown in Figure~\ref{fig:st_gpvae}, demonstrating how the states for each spatial location are independently filtered and smoothed over time, and then spatially mixed by the emission matrix. Lastly, we approximate the likelihood with a mean-field amortised approximation $\prod_{t=1}^T N(\tilde{\bY}_{t}|[\mathbf{L}_{\mathbf{R}\mathbf{R}}^{r}\otimes \bH^{t}] \bs_t, \tilde{\bV}_t)$, where $\tilde{\bY}_{t},\tilde{\bV}_{t}\in\mathbb{R}^{N_x}$. See Appendix~\ref{appendix:spatiotemporal} for further details. 

\begin{figure}[t]
\centering
\vspace{0mm}
\begin{minipage}{0.45\textwidth}
\centering
\includegraphics[width = \columnwidth]{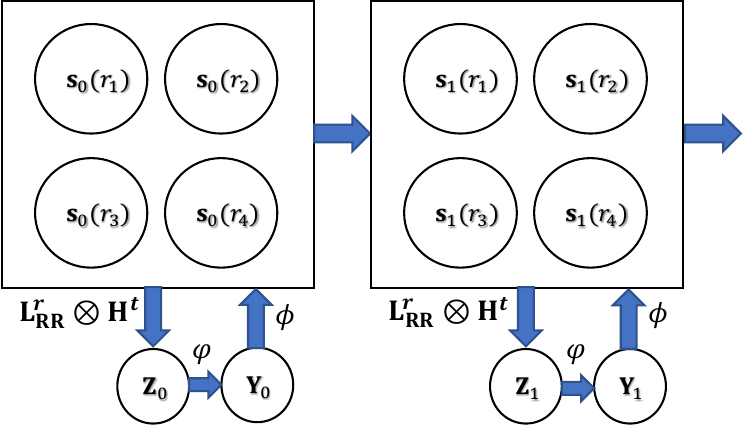}
\vspace{-4mm}
\end{minipage}
\vspace{-3mm}
\caption{Graphical model of the separable spatiotemporal MGPVAE model. The temporal dynamics of the states $\bs_t(r)$ at each location $r$ are independently handled. At each time step $t$, these states are spatially mixed to produce $\bZ_t$, which is then transformed by a non-linear mapping to $\bY_t$. }
\vspace{-4mm}
\label{fig:st_gpvae}
\end{figure}

\subsection{Computational Complexity and Storage}
\label{subsection:computational}
\paragraph{Computational Complexity: }For simplicity let us assume that each channel has the same kernel but is modelled independently. Then the computational complexity of MGPVAE is $\mathcal{O}(Ld^3T)$. For GPVAE, it is also $\mathcal{O}(Ld^3T)$; for SVGPVAE, $\mathcal{O}(L(Tm^2 + m^3))$ with $m$ inducing points. For KVAE, VRNN and CRU, the complexity is $\mathcal{O}(LT)$, but there may be large big-O constants due to the RNN network sizes. For neuralODEs and neuralSDEs, the complexity is linear with respect to the number of discretisation steps, which can potentially be much larger than $T$. 

For spatiotemporal modelling, the computational complexity for MGPVAE will be $\mathcal{O}(Ld^3 T N_r^3)$ in this case, but can be further lowered to $\mathcal{O}(Ld^3 T (N_rm^2 + m^3))$ if we sparsify the spatial domain by using $M$ spatial inducing points. For SVGPVAE, it is $\mathcal{O}(L(TN_rm^2 + m^3))$ with $m$ inducing points over space and time.

\paragraph{Storage: } The storage requirements for MGPVAE and SVGPVAE are $\mathcal{O}(Td^2)$ and $\mathcal{O}(Tm + m^2)$ respectively. For larger $T$, $m$ needs to be larger and hence $m \gg d^2$ in many cases (e.g. $d^2=4$ for the Matern-$\sfrac{3}{2}$ kernel).

\section{Related Work}

GPVAEs have been explored for time series modelling in \citet{fortuin_gp-vae_2020, ashman_sparse_2020}. In our work, we experiment with the state-of-the-art GPVAEs of \citet{fortuin_gp-vae_2020} and SVGPVAE \citep{jazbec_scalable_2021}, and explore a wider variety of tasks. Unlike \citet{fortuin_gp-vae_2020}, MGPVAE is capable of tackling both corrupt and missing frames imputation tasks, as well as spatiotemporal modelling tasks, with great scalability. Compared to sparse GPVAEs \citep{ashman_sparse_2020, jazbec_scalable_2021}, MGPVAE does not require inducing points.

NeuralODEs \citep{chen2018neural, rubanova2019latent} and neuralSDEs \citep{li2020scalable} are also continuous-time models that can tackle the same tasks as MGPVAE. However, they depend heavily on the time discretisation, how well the initial conditions are learned and the expressiveness of the drift and diffusion functions. In our work, we experiment with neuralODEs, which have previously been used for similar missing frames imputation tasks, and find that it is the slowest model whithout achieving good predictive performance.

Many discrete-time and continuous-time models, such as VRNN \citep{chung2015recurrent}, KVAE \citep{fraccaro2017disentangled} and latentODE \citep{rubanova2019latent}, are only designed to model temporal, but not spatiotemporal, data. On the other hand, classical multioutput GPs can only handle lower-dimensional spatiotemporal datasets as there cannot be any dimensionality reduction to a latent space, whereas GPVAE enables the encoder-decoder networks to learn meaningful low-dimensional representations for high-dimensional data. SVGPVAE is able to handle spatiotemporal data, but has the disadvantage of being less efficient than MGPVAE due to the use of inducing points over space and time jointly. 

\citet{chang_fast_2020, hamelijnck_spatio-temporal_2021} considered modelling non-conjugate likelihoods with Gaussian approximations, which would allow for Kalman filtering and smoothing operations. We adopt this strategy to allow flexible decoder choices with non-linear mapping, where a key difference is that the encoder helps us construct a low-dimensional approximation to the likelihood function, which allows us to work with Kalman filtering and smoothing in the lower-dimensional latent space.

\section{Experiments}
We present 3 sets of experiments: rotating MNIST, Mujoco action data and spatiotemporal data modelling. We benchmark MGPVAE against a variety of continuous and discrete time models, such as GPVAE \citep{fortuin_gp-vae_2020}, SVGPVAE \citep{jazbec_scalable_2021}, KVAE \citep{fraccaro2017disentangled}, VRNN \citep{chung2015recurrent}, LatentODE \citep{rubanova2019latent}, CRU (\citet{schirmer2022modeling}; only report RMSE as it was not originally conceived as a generative model) and sparse variational multioutput GP (MOGP). We evaluate the performances using both test negative log-likelihood (NLL) and root mean squared error (RMSE).  We implemented each model across different libraries (JAX, PyTorch and TensorFlow) due to varying suitabilities and tried our best to optimise each implementation for fairness of comparison. All wall-clock time computations are done on NVIDIA RTX-3090 GPUs with 24576MiB RAM. See further experimental details and results in Appendix~\ref{appendix:experiments}.

\subsection{Rotating MNIST}
In this experiment, we produce sequences of MNIST frames in which the digits are rotated with a periodic length of 50, over $T=100$ frames. We tackle 2 imputation tasks for: corrupted frames where the frame pixels are randomly set to 0, and missing frames where frames are randomly dropped out of each sequence. Each task has 4000 /1000 train/test sequences respectively. The underlying dynamics are simple (rotation) which may favour models with stronger inductive biases, such as GPVAE and MGPVAE. To test this, for both models, we use Matern-$\sfrac{3}{2}$ kernels with lengthscales initialised at 40 (fixed for GPVAE according to \citet{fortuin_gp-vae_2020}).

\begin{figure*}[t]
\centering
\vspace{-3mm}
\begin{minipage}{\textwidth}
\centering
\includegraphics[width = \columnwidth]{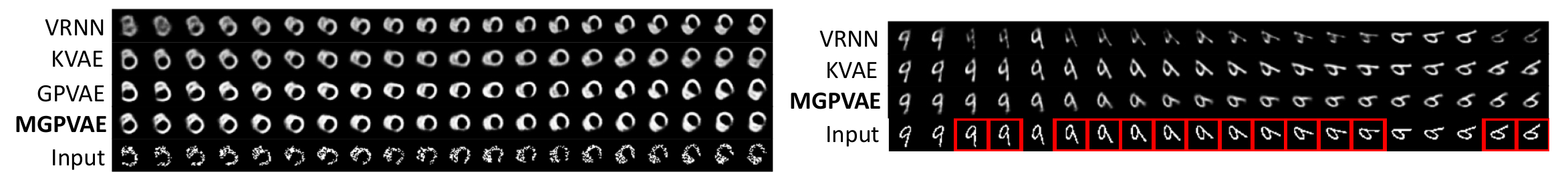}
\vspace{-7mm}
\caption{(Left) Corrupt frames imputation results for an unseen sequence of 5's. (Right) Missing frames imputation results for an unseen sequence of 9's. Missing frames are \textcolor{red}{red frames}.}
\label{fig:hmnist}
\end{minipage}
\begin{minipage}{\textwidth}
\centering
\captionof{table}{Test NLL and RMSE for both the corrupt (Cor) and missing frames (Mis) imputation tasks.}
\vspace{-3mm}
\begin{adjustbox}{width=\linewidth,center}
\begin{tabular}[t]{| c | c |c |c |c |c| c| c|}
\hline
Model  & NLL-Cor $(\downarrow)$ & RMSE-Cor  $(\downarrow)$& Time-Cor (s/epoch $\downarrow$) & NLL-Mis $(\downarrow)$ & RMSE-Mis $(\downarrow)$ & Time-Mis (s/epoch $\downarrow$) \\
\hline
\hline
VRNN & $9898\pm162.0$ & $0.1768\pm0.001563$ & 63.51 & $16240\pm2090$ & $0.1796\pm0.008002$ & 103.6 \\
KVAE & $12500\pm83.13$  & $0.2025\pm0.0006077$ &  139.2 &$10730\pm1232$ & $0.1582\pm0.008688$ & 149.0\\
GPVAE &$9026\pm 48.70$ & $\mathbf{0.1340\pm 0.0004529}$ & \textbf{48.93} & NA & NA & NA
 \\
\textbf{MGPVAE} & $\mathbf{8556\pm69.66}$ & $\mathbf{0.1468\pm0.0006738}$ & \textbf{50.45} & $\mathbf{8925\pm53.40}$ & $\mathbf{0.1508\pm0.0005190}$ & \textbf{59.43} \\
\hline
\end{tabular}
\end{adjustbox}
\label{table:hmnist}
\end{minipage}
\vspace{-4mm}
\end{figure*}

\paragraph{Corrupt frames imputation:} This is a task that highly suits standard RNN-based models such as VRNN and KVAE, since the frames are observed at regular time steps. Even so, we see from Table~\ref{table:hmnist} that both GPVAE and MGPVAE perform significantly better than VRNN and KVAE in both NLL and RMSE, validating our hypothesis of inductive biases helping the learning dynamics. Furthermore, we observe that the RMSE for MGPVAE is worse than GPVAE, although it has a slightly better NLL. However, the left panel of Figure~\ref{fig:hmnist} shows that the images generated do not visually differ significantly, which implies that the model performance is comparable.

\paragraph{Missing frames imputation:} \citet{rubanova2019latent} showed that RNN-based models fail in imputing irregularly sampled time-series, as they struggle to correctly update the hidden states at time steps of unobserved frames. This is confirmed by our results in Table~\ref{table:hmnist}: MGPVAE outperforms both VRNN and KVAE in terms of NLL and RMSE.\footnote{We omit GPVAE here as the implementation by \citet{fortuin_gp-vae_2020} cannot efficiently handle missing frames in batches. See Appendix~\ref{appendix:experiments}.} 
In the right panel of Figure~\ref{fig:hmnist}, we illustrate the posterior mean imputations, and again VRNN fails. These results are expected since VRNN implements filtering (only includes past observations), while MGPVAE and KVAE include a smoothing step (includes both past and future observations).

\subsection{Mujoco Action Data}
The Mujoco dataset is a physical simulation dataset generated using the Deepmind Control Suite \citep{tunyasuvunakool2020}. We obtained the Hopper generation code from \citet{rubanova2019latent}, which outputs 14 dimensions sequences, and for all models we use 15-dimensional latent dimensions (according to \citet{rubanova2019latent}). We modify the task so that we only train on the observed time steps, whereas in \citet{rubanova2019latent} the models have access to data at all the time steps. This makes the task harder as the model has less information to work with during training. We have 2 settings (1) 1280/400 train-test split with length $T=100$ and (2) 320/100 for length $T=1000$. Compared to rotating MNIST, the underlying nonlinear dynamics are more complex, and each dimension can behave differently. For both SVGPVAE and MGPVAE, we use Matern-$\sfrac{3}{2}$ kernels.

\paragraph{Results:} We see from Table~\ref{table:mujoco} that the performance of VRNN, KVAE, latentODE and CRU are significantly worse than GP-based models, and overall both SVGPVAE and MGPVAE achieve the best NLL and RMSE. This is because missing data imputation is a difficult task for the discrete-time RNN-based models. On the other hand, latentODE significantly underperforms, possibly due to the ELBO being computed only over the observed time steps, making it more difficult to fit the model than the original task in \citet{rubanova2019latent}. Figure~\ref{fig:mujoco} confirms the results; indeed the non-GP models struggle to simultaneously fit the data and estimate uncertainty well. Additional results can be found in Appendix~\ref{appendix:action}. 

In terms of time complexity, VRNN, KVAE, latentODE and CRU are comparable to each other as shown in Figure~\ref{fig:timing}.  The time and memory complexities for SVGPVAE is dependent on the number of inducing points (see section~\ref{subsection:computational}) and there is a trade-off between model expressiveness and inducing points. For longer sequences, such as for $T=1000$, we would have to choose more inducing points to gain comparable performance to MGPVAE. We see that with only 20 inducing points, although the wall-clock time is faster than MGPVAE, SVGPVAE-20 underperforms MGPVAE; with 40 inducing points, although the performance and time complexities are comparable to MGPVAE, it also has maxed-out the memory on our GPU. In comparison, MGPVAE does not require any inducing points and is the most time-efficient model.

\begin{figure*}[t]
\centering
\begin{minipage}{\textwidth}
\includegraphics[width = \textwidth]{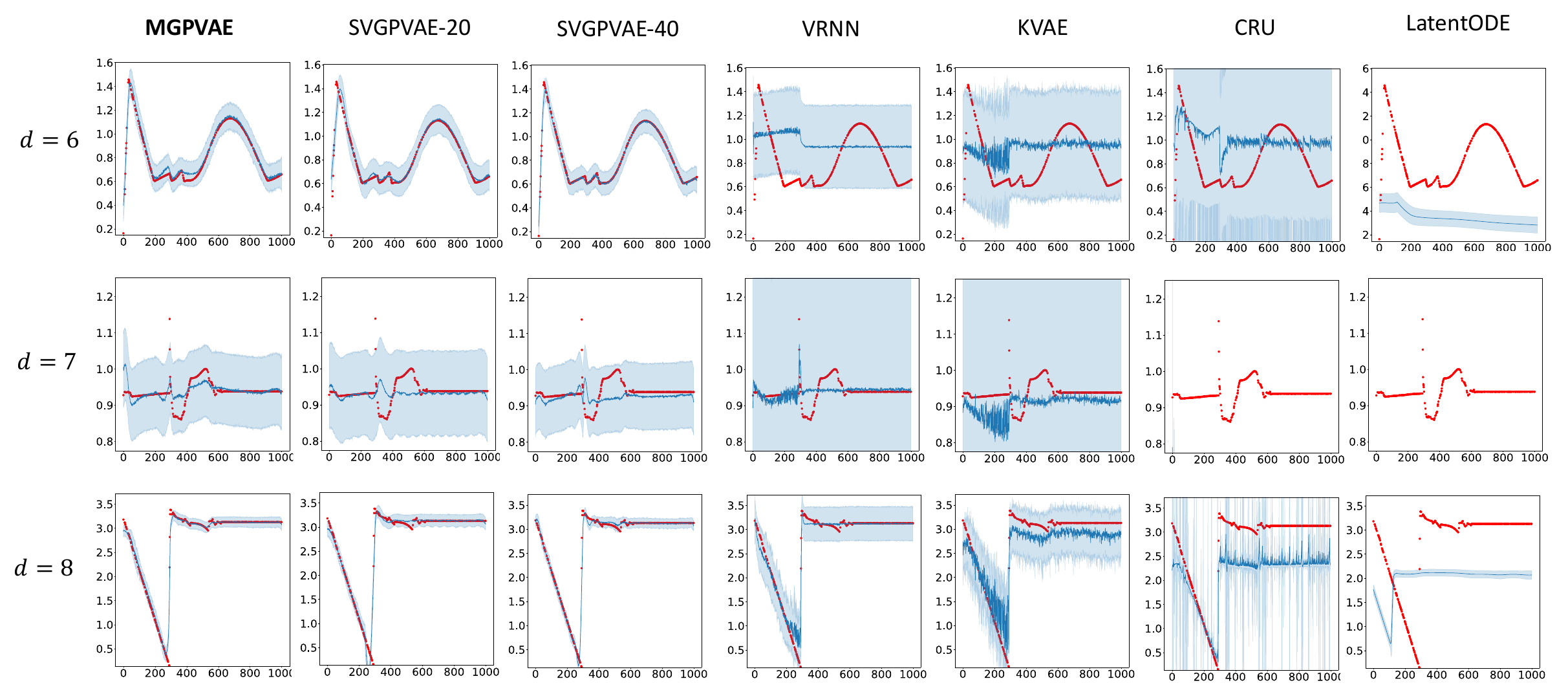}
\vspace{-7mm}
\caption{95$\%$ posterior credible intervals for unseen mujoco sequences in its 6,7 and 8th dimensions with $T=1000$. The red dots show observed data. Note that some predictions are not showing as they fall outside the limits.}
\label{fig:mujoco}
\end{minipage}
\end{figure*}

\begin{figure}[t]
\centering
\begin{minipage}{\columnwidth}
\includegraphics[width = \columnwidth]{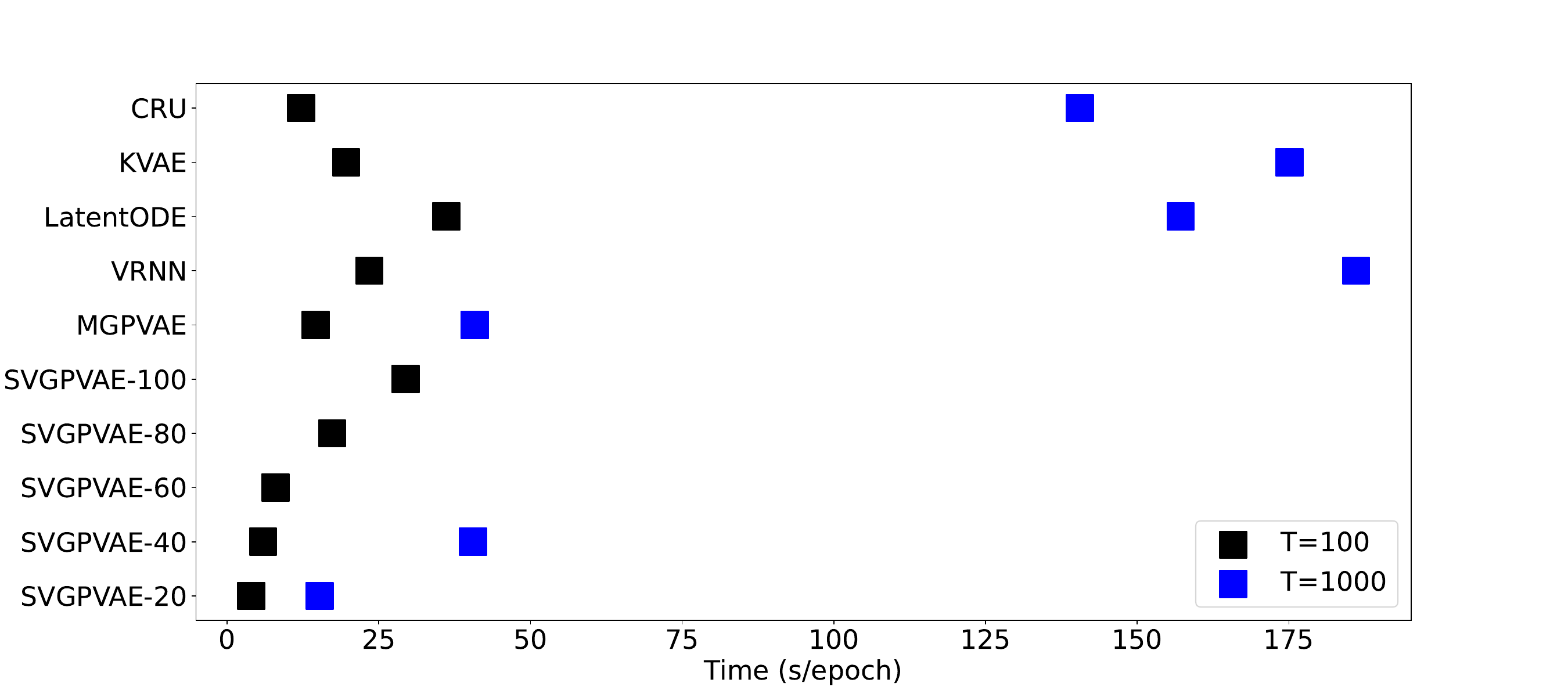}
\vspace{-7mm}
\caption{Wallclock times for each model in the Mujoco experiment.}
\label{fig:timing}
\end{minipage}
\end{figure}

\begin{figure}[t]
\vspace{-3mm}
\begin{minipage}{0.48\textwidth}
\centering
\captionof{table}{Imputation results for the Mujoco tasks.}
\vspace{-4mm}
\begin{adjustbox}{width=\linewidth,center}
\begin{tabular}[t]{| c | c | c | c | c | }
\hline
Model & NLL ($T=100$) $(\downarrow)$ & RMSE ($T=100$) $(\downarrow)$ & NLL ($T=1000$) $(\downarrow)$ & RMSE ($T=1000$) $(\downarrow)$ \\
\hline
\hline
CRU & - & $0.1343\pm0.009169$ & - &  $0.1353\pm0.008574 $\\
\hline 
VRNN & $-385.2\pm 25.59$ & $0.1774\pm0.002499$ & $-2877\pm 450.8$ &  $0.1844\pm 0.004053$\\
KVAE & $-8.353\pm 25.6$ & $0.1828\pm 0.004587$ &  $-262.4\pm 1141$& $0.1761\pm 0.01751$\\
LatentODE & $124\pm 11.99$ & $0.06599\pm 0.0007146$ & $240.6\pm 38.62$ & $0.07749\pm 0.001119$\\
SVGPVAE-20 & $\mathbf{-2438\pm 111}$ & $0.02841\pm 0.003288$ & $-18020\pm 282.6$ & $0.0538\pm 0.0007901$\\
SVGPVAE-40 & $\mathbf{-2468\pm 106.4}$ & $\mathbf{0.02566\pm 0.002945}$ & $-21290\pm 136.8$ & $0.04237\pm 0.000295$\\
SVGPVAE-60 & $-2290\pm 53.69$ & $0.03014\pm 0.001579$ & - & -\\
SVGPVAE-80 & $-2312\pm 67.76$ & $0.0287\pm 0.001866$ &- & -\\
\textbf{MGPVAE} & $-2292\pm 18.02$ & $0.03068\pm0.0006991$ &$\mathbf{-21610\pm 233.7}$ & $\mathbf{0.04156\pm 0.0007136}$\\
\hline
\end{tabular}
\end{adjustbox}
\label{table:mujoco}
\end{minipage}
\vspace{-5mm}
\end{figure}

\subsection{Spatiotemporal Climate Data}
We obtained climate data, including temperature and precipitation, from ERA5 using Google Earth Engine API \citep{gorelick2017google}. The task is to condition on observed data $\{\bY(r,t)\}_{r,t}$ (temperature and air pressure), and predict on unknown spatial locations $\{\bY(r_*,t)\}_{r_*,t}$ for all time steps. Here we focus on GP-based models since they, unlike RNN-based models, can flexibly handle spatiotemporal data by combining spatial and temporal kernels to efficiently model correlated structures. In particular, for GPVAE models we use a separable kernel $k(r,t,r',t)=k_r(r,r')k_t(t,t')$ for each latent channel and the spatial kernel $k_r$ is shared across channels. We also consider Gaussian process prediction which is also known as kriging in spatial statistics \citep{cressie2015statistics}. Sparse approximation is needed for scalability for the MOGP baseline due to computational feasibility. We emphasise that the dimensionality of this problem, which is 8, is high relative to traditional spatiotemporal modelling problems \citep{cressie2015statistics}.

\begin{figure*}[t]
\centering
\vspace{-3mm}
\begin{minipage}{\textwidth}
\centering
\includegraphics[width = \columnwidth]{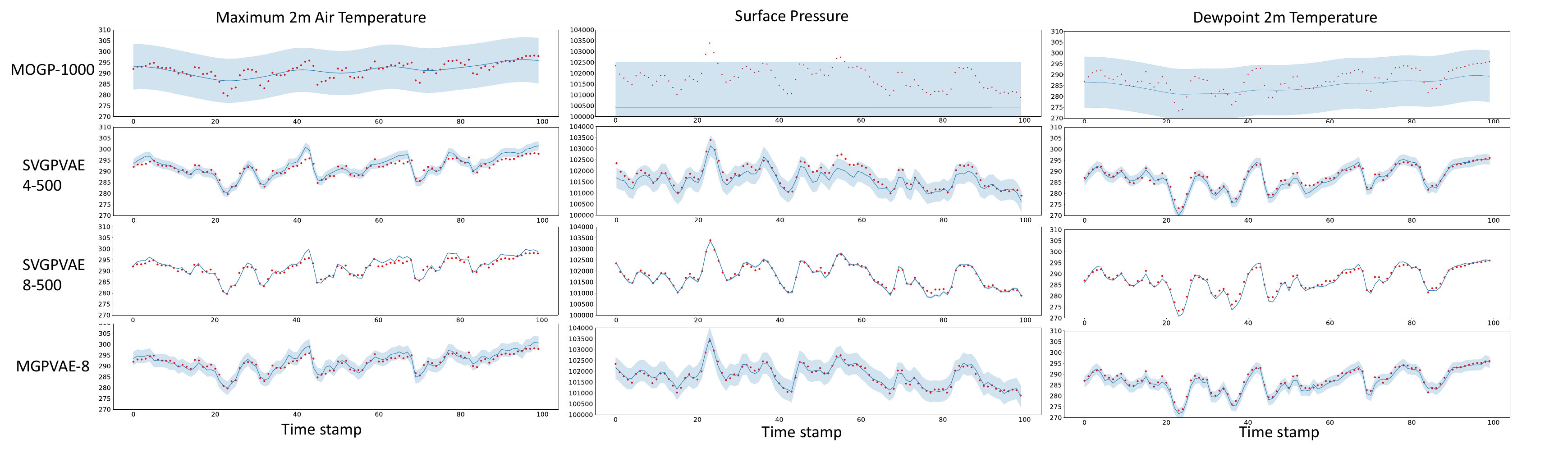}
\vspace{-8mm}
\caption{95$\%$ posterior credible intervals for climate variables at an unseen spatial location. }
\label{fig:era5_time}
\vspace{-4mm}
\end{minipage}
\end{figure*}

\begin{figure*}[tbp]
\centering
\includegraphics[width = \textwidth]{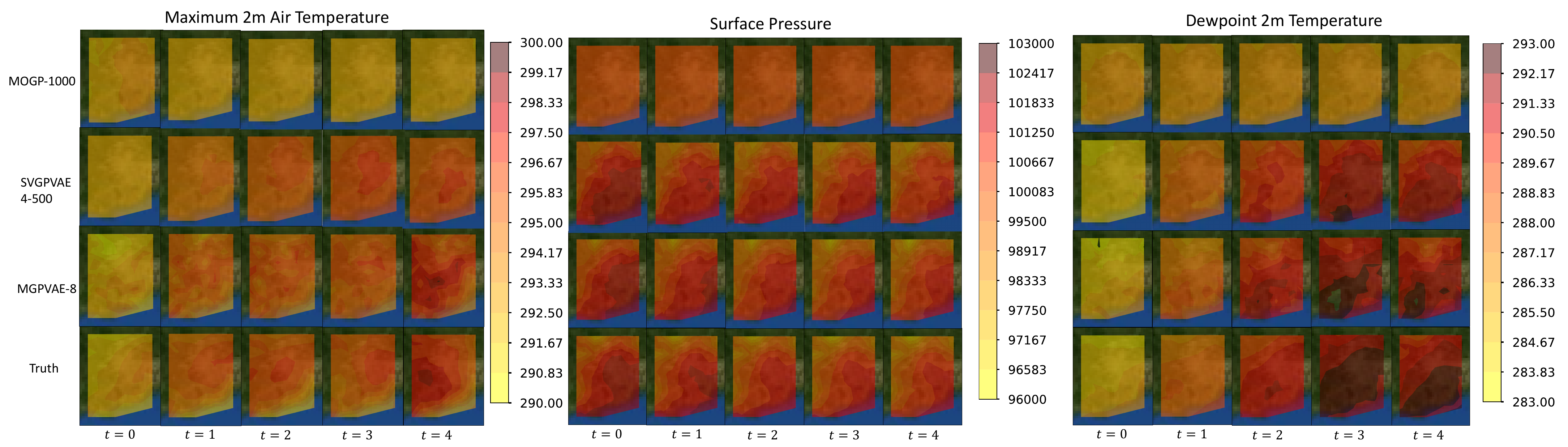}
\vspace{-8mm}
\caption{Posterior means over space for climate variables over 5 time steps.}
\vspace{-3mm}
\label{fig:era5_spatial}
\end{figure*}

\paragraph{Results:} We report the corresponding quantitative results in Table~\ref{table:era5} and visualise the prediction results at an unseen spatial location in Figure~\ref{fig:era5_time}. 
We observe that MOGP underfits and overestimates the uncertainty, which is expected as traditional GP regression models struggle to fit a complicated multi-dimensional time series with non-warped GPs (without decoder network). In comparison, the encoder-decoder networks allow GPVAEs to learn simpler dynamics, resulting in better performance. The conclusions are similar when we fix the time steps and plot the corresponding posterior mean over space in Figure~\ref{fig:era5_spatial}, where GPVAE models better predict the spatial patterns.

SVGPVAE underperforms MGPVAE even when pushing the number of inducing points to the memory limit of our hardware (500). Interestingly, SVGPVAE underestimates the uncertainty more as we increase the number of latent channels to 8, as can also be seen from the large negative log-likelihood per dim (NLPD) values. It is unclear why this occurs, though this could be related to an imbalance in the regularisation effect with the KL divergence. We note that MGPVAE does not suffer from the same issue. 

SVGPVAE results can be improved if having well-placed and sufficiently-many inducing points, but this also means SVGPVAE is memory inefficient when compared with MGPVAE. For the spatiotemporal modelling task, a significantly larger number of inducing points is required and thus further illustrates the drawbacks of SVGPVAE for spatiotemporal modelling. Similarly for computational time, MGPVAE is faster than the other models with more inducing points, but slower than the ones with less of inducing points (which also perform worse).

\begin{figure}[tbp]
\centering
\begin{minipage}{0.5\textwidth}
\centering
\captionof{table}{ERA5 prediction results.}
\vspace{-3mm}
\begin{adjustbox}{width=0.9\linewidth,center}
\begin{tabular}[t]{| c | c | c| c|}
\hline
Model  & NLPD $(\downarrow)$  & RMSE $(\downarrow)$ & Time $(\text{s/epoch } \downarrow)$ \\
\hline
\hline
MOGP-100 &  $10.38\pm0.1783$  &$1119\pm21.29$ & 0.1220 \\
MOGP-1000 &  $10.38\pm0.1784$  & $1119\pm21.13$ & 0.6001\\
SVGPVAE-4-100  & $16.40\pm1.531$  & $434.04\pm11.40$ & 0.05987 \\
SVGPVAE-4-500 & $8.222\pm0.9483$  & $388.8\pm9.347$ & 0.7046 \\
SVGPVAE-8-100 & $1376\pm473.0$  & $392.2\pm 18.08$& 0.08675 \\
SVGPVAE-8-500 & $2613\pm 540.6$  & $\mathbf{313.9\pm21.42}$ & 1.344 \\
MGPVAE-4 & $2.454\pm0.1149$  & $402.1\pm 15.24$ & 0.45876 \\
\textbf{MGPVAE-8} & $\mathbf{-0.3070\pm0.1021}$ & $\mathbf{352.7\pm24.86}$ & 0.5128 \\
\hline
\end{tabular}
\end{adjustbox}
\label{table:era5}
\end{minipage}
\vspace{-5mm}
\end{figure}

\section{Conclusion and Discussion}
We propose MGPVAE, a GPVAE model using Markovian Gaussian processes that uses Kalman filtering and smoothing with linear time complexity. This is achieved by approximating the non-Gaussian likelihoods using Gaussian sites. Compared to \citet{fortuin_gp-vae_2020}, which also achieves linear time complexity by using an approximate covariance structure, our model leaves the original Gaussian process covariance structure intact, and additionally work with both irregularly-sampled time series and spatiotemporal data. Experiments on video, Mujoco action and climate modelling tasks show that our method is both competitive and scalable, compared to modern discrete and continuous-time models. Future work can explore the use of parallel filtering \citep{sarkka2020temporal}, other nonlinear filtering approaches \citep{kamthe2022iterative} and forecasting applications.

\paragraph{Limitations:} MGPVAE requires the use of kernels that admit a Markovian decomposition, and therefore kernels such that the squared exponential kernel will not be permissible \citep{sarkka2019applied} without additional approximations. However, we argue that most commonly used kernels are indeed Markovian, which should be sufficient for most applications. In addition, we need to use Gaussian site approximations for the likelihood, in order to allow for Kalman filtering and smoothing. This may lower the approximation accuracy, though in our experiments we see that MGPVAE still performs strongly.

\section{Acknowledgements}
We would like to especially thank Wenlin Chen for his valuable help with code and experiments during the revision period, especially implementing a more optimised version of SVGPVAE with functorch \citep{functorch2021} compared to the original TensorFlow implementation of \citet{jazbec_scalable_2021}. HZ was supported by the EPSRC Centre for Doctoral Training in Modern Statistics and Statistical Machine Learning (EP/S023151/1) and the Department of Mathematics of Imperial College London. HZ was supported by Cervest Limited. We would also like to thank Andy Thomas for his endless support with using the NVIDIA4, Forrest and NVIDIA6 GPU Compute Servers.

{ 
\bibliography{ref}

\begin{thebibliography}{57}
\providecommand{\natexlab}[1]{#1}
\providecommand{\url}[1]{\texttt{#1}}
\expandafter\ifx\csname urlstyle\endcsname\relax
  \providecommand{\doi}[1]{doi: #1}\else
  \providecommand{\doi}{doi: \begingroup \urlstyle{rm}\Url}\fi

\bibitem[Adam et~al.(2020)Adam, Eleftheriadis, Artemev, Durrande, and
  Hensman]{adam_doubly_2020}
Adam, V., Eleftheriadis, S., Artemev, A., Durrande, N., and Hensman, J.
\newblock Doubly sparse variational {Gaussian} processes.
\newblock In \emph{International {Conference} on {Artificial} {Intelligence}
  and {Statistics}}, pp.\  2874--2884. PMLR, 2020.

\bibitem[Ashman et~al.(2020)Ashman, So, Tebbutt, Fortuin, Pearce, and
  Turner]{ashman_sparse_2020}
Ashman, M., So, J., Tebbutt, W., Fortuin, V., Pearce, M., and Turner, R.~E.
\newblock Sparse {Gaussian} process variational autoencoders.
\newblock \emph{arXiv preprint arXiv:2010.10177}, 2020.

\bibitem[Bamler \& Mandt(2017)Bamler and Mandt]{bamler2017dynamic}
Bamler, R. and Mandt, S.
\newblock Dynamic word embeddings.
\newblock In \emph{International conference on Machine learning}, pp.\
  380--389. PMLR, 2017.

\bibitem[Blei \& Lafferty(2006)Blei and Lafferty]{blei2006dynamic}
Blei, D.~M. and Lafferty, J.~D.
\newblock Dynamic topic models.
\newblock In \emph{Proceedings of the 23rd international conference on Machine
  learning}, pp.\  113--120, 2006.

\bibitem[Burt et~al.(2019)Burt, Rasmussen, and Van Der~Wilk]{burt2019rates}
Burt, D., Rasmussen, C.~E., and Van Der~Wilk, M.
\newblock Rates of convergence for sparse variational gaussian process
  regression.
\newblock In \emph{International Conference on Machine Learning}, pp.\
  862--871. PMLR, 2019.

\bibitem[Casale et~al.(2018)Casale, Dalca, Saglietti, Listgarten, and
  Fusi]{casale_gaussian_2018}
Casale, F.~P., Dalca, A., Saglietti, L., Listgarten, J., and Fusi, N.
\newblock Gaussian process prior variational autoencoders.
\newblock \emph{Advances in neural information processing systems}, 31, 2018.

\bibitem[Chang et~al.(2020)Chang, Wilkinson, Khan, and Solin]{chang_fast_2020}
Chang, P.~E., Wilkinson, W.~J., Khan, M.~E., and Solin, A.
\newblock Fast variational learning in state-space {Gaussian} process models.
\newblock In \emph{2020 {IEEE} 30th {International} {Workshop} on {Machine}
  {Learning} for {Signal} {Processing} ({MLSP})}, pp.\  1--6. IEEE, 2020.

\bibitem[Chen et~al.(2018)Chen, Rubanova, Bettencourt, and
  Duvenaud]{chen2018neural}
Chen, R.~T., Rubanova, Y., Bettencourt, J., and Duvenaud, D.~K.
\newblock Neural ordinary differential equations.
\newblock \emph{Advances in neural information processing systems}, 31, 2018.

\bibitem[Chen et~al.(2017)Chen, Kingma, Salimans, Duan, Dhariwal, Schulman,
  Sutskever, and Abbeel]{chen2017variational}
Chen, X., Kingma, D.~P., Salimans, T., Duan, Y., Dhariwal, P., Schulman, J.,
  Sutskever, I., and Abbeel, P.
\newblock Variational lossy autoencoder.
\newblock \emph{ICLR}, 2017.

\bibitem[Cho et~al.(2014)Cho, van Merri{\"e}nboer, Bahdanau, and
  Bengio]{cho-etal-2014-properties}
Cho, K., van Merri{\"e}nboer, B., Bahdanau, D., and Bengio, Y.
\newblock On the properties of neural machine translation: Encoder{--}decoder
  approaches.
\newblock In \emph{Proceedings of {SSST}-8, Eighth Workshop on Syntax,
  Semantics and Structure in Statistical Translation}, pp.\  103--111, Doha,
  Qatar, October 2014. Association for Computational Linguistics.
\newblock \doi{10.3115/v1/W14-4012}.
\newblock URL \url{https://aclanthology.org/W14-4012}.

\bibitem[Chung et~al.(2015)Chung, Kastner, Dinh, Goel, Courville, and
  Bengio]{chung2015recurrent}
Chung, J., Kastner, K., Dinh, L., Goel, K., Courville, A.~C., and Bengio, Y.
\newblock A recurrent latent variable model for sequential data.
\newblock \emph{Advances in neural information processing systems}, 28, 2015.

\bibitem[Cressie(2015)]{cressie2015statistics}
Cressie, N.
\newblock \emph{Statistics for spatial data}.
\newblock John Wiley \& Sons, 2015.

\bibitem[Developers(2020)]{objax2020github}
Developers, O.
\newblock Objax, 2020.
\newblock URL \url{https://github.com/google/objax}.

\bibitem[Evans(2006)]{evans_introduction_2006}
Evans, L.~C.
\newblock An introduction to stochastic differential equations version 1.2.
\newblock \emph{Lecture Notes, UC Berkeley}, 2006.

\bibitem[Fortuin et~al.(2020)Fortuin, Baranchuk, Rätsch, and
  Mandt]{fortuin_gp-vae_2020}
Fortuin, V., Baranchuk, D., Rätsch, G., and Mandt, S.
\newblock Gp-vae: {Deep} probabilistic time series imputation.
\newblock In \emph{International conference on artificial intelligence and
  statistics}, pp.\  1651--1661. PMLR, 2020.

\bibitem[Fraccaro et~al.(2017)Fraccaro, Kamronn, Paquet, and
  Winther]{fraccaro2017disentangled}
Fraccaro, M., Kamronn, S., Paquet, U., and Winther, O.
\newblock A disentangled recognition and nonlinear dynamics model for
  unsupervised learning.
\newblock \emph{Advances in neural information processing systems}, 30, 2017.

\bibitem[Goel et~al.(2022)Goel, Gu, Donahue, and R{\'e}]{goel2022s}
Goel, K., Gu, A., Donahue, C., and R{\'e}, C.
\newblock It’s raw! audio generation with state-space models.
\newblock In \emph{International Conference on Machine Learning}, pp.\
  7616--7633. PMLR, 2022.

\bibitem[Gorelick et~al.(2017)Gorelick, Hancher, Dixon, Ilyushchenko, Thau, and
  Moore]{gorelick2017google}
Gorelick, N., Hancher, M., Dixon, M., Ilyushchenko, S., Thau, D., and Moore, R.
\newblock Google earth engine: Planetary-scale geospatial analysis for
  everyone.
\newblock \emph{Remote Sensing of Environment}, 2017.
\newblock \doi{10.1016/j.rse.2017.06.031}.
\newblock URL \url{https://doi.org/10.1016/j.rse.2017.06.031}.

\bibitem[Gu et~al.(2021)Gu, Goel, and R{\'e}]{gu2021efficiently}
Gu, A., Goel, K., and R{\'e}, C.
\newblock Efficiently modeling long sequences with structured state spaces.
\newblock \emph{arXiv preprint arXiv:2111.00396}, 2021.

\bibitem[Hamelijnck et~al.(2021)Hamelijnck, Wilkinson, Loppi, Solin, and
  Damoulas]{hamelijnck_spatio-temporal_2021}
Hamelijnck, O., Wilkinson, W., Loppi, N., Solin, A., and Damoulas, T.
\newblock Spatio-temporal variational {Gaussian} processes.
\newblock \emph{Advances in Neural Information Processing Systems}, 34, 2021.

\bibitem[Hensman et~al.(2013)Hensman, Fusi, and Lawrence]{hensman2013gaussian}
Hensman, J., Fusi, N., and Lawrence, N.~D.
\newblock Gaussian processes for big data.
\newblock \emph{UAI}, 2013.

\bibitem[Hensman et~al.(2015)Hensman, Matthews, and
  Ghahramani]{hensman2015scalable}
Hensman, J., Matthews, A., and Ghahramani, Z.
\newblock Scalable variational gaussian process classification.
\newblock In \emph{Artificial Intelligence and Statistics}, pp.\  351--360.
  PMLR, 2015.

\bibitem[Hochreiter \& Schmidhuber(1997)Hochreiter and
  Schmidhuber]{hochreiter1997long}
Hochreiter, S. and Schmidhuber, J.
\newblock Long short-term memory.
\newblock \emph{Neural computation}, 9\penalty0 (8):\penalty0 1735--1780, 1997.

\bibitem[Horace~He(2021)]{functorch2021}
Horace~He, R.~Z.
\newblock functorch: Jax-like composable function transforms for pytorch.
\newblock \url{https://github.com/pytorch/functorch}, 2021.

\bibitem[Jazbec et~al.(2021)Jazbec, Ashman, Fortuin, Pearce, Mandt, and
  Rätsch]{jazbec_scalable_2021}
Jazbec, M., Ashman, M., Fortuin, V., Pearce, M., Mandt, S., and Rätsch, G.
\newblock Scalable gaussian process variational autoencoders.
\newblock In \emph{International {Conference} on {Artificial} {Intelligence}
  and {Statistics}}, pp.\  3511--3519. PMLR, 2021.

\bibitem[Kamthe et~al.(2022)Kamthe, Takao, Mohamed, and
  Deisenroth]{kamthe2022iterative}
Kamthe, S., Takao, S., Mohamed, S., and Deisenroth, M.
\newblock Iterative state estimation in non-linear dynamical systems using
  approximate expectation propagation.
\newblock \emph{Transactions on Machine Learning Research}, 2022.

\bibitem[Kanagawa et~al.(2018)Kanagawa, Hennig, Sejdinovic, and
  Sriperumbudur]{kanagawa_gaussian_2018}
Kanagawa, M., Hennig, P., Sejdinovic, D., and Sriperumbudur, B.~K.
\newblock Gaussian {Processes} and {Kernel} {Methods}: {A} {Review} on
  {Connections} and {Equivalences}.
\newblock \emph{arXiv:1807.02582 [cs, stat]}, July 2018.
\newblock URL \url{http://arxiv.org/abs/1807.02582}.
\newblock arXiv: 1807.02582.

\bibitem[Khan \& Lin(2017)Khan and Lin]{khan2017conjugate}
Khan, M. and Lin, W.
\newblock Conjugate-computation variational inference: Converting variational
  inference in non-conjugate models to inferences in conjugate models.
\newblock In \emph{Artificial Intelligence and Statistics}, pp.\  878--887.
  PMLR, 2017.

\bibitem[Kidger(2022)]{kidger2022neural}
Kidger, P.
\newblock On neural differential equations.
\newblock \emph{PhD Thesis}, 2022.

\bibitem[Kingma \& Welling(2014)Kingma and Welling]{kingma_auto-encoding_2014}
Kingma, D.~P. and Welling, M.
\newblock Auto-{Encoding} {Variational} {Bayes}.
\newblock \emph{arXiv:1312.6114 [cs, stat]}, May 2014.
\newblock URL \url{http://arxiv.org/abs/1312.6114}.
\newblock arXiv: 1312.6114.

\bibitem[Lee et~al.(2019)Lee, Lee, Kim, Kosiorek, Choi, and Teh]{lee2019set}
Lee, J., Lee, Y., Kim, J., Kosiorek, A., Choi, S., and Teh, Y.~W.
\newblock Set transformer: A framework for attention-based
  permutation-invariant neural networks.
\newblock In \emph{International conference on machine learning}, pp.\
  3744--3753. PMLR, 2019.

\bibitem[Li et~al.(2020)Li, Wong, Chen, and Duvenaud]{li2020scalable}
Li, X., Wong, T.-K.~L., Chen, R.~T., and Duvenaud, D.
\newblock Scalable gradients for stochastic differential equations.
\newblock In \emph{International Conference on Artificial Intelligence and
  Statistics}, pp.\  3870--3882. PMLR, 2020.

\bibitem[Li \& Mandt(2018)Li and Mandt]{li2018disentangled}
Li, Y. and Mandt, S.
\newblock Disentangled sequential autoencoder.
\newblock \emph{ICML}, 2018.

\bibitem[Lin \& Yin(2023)Lin and Yin]{lin2023towards}
Lin, Z. and Yin, F.
\newblock Towards flexibility and interpretability of gaussian process
  state-space model.
\newblock \emph{arXiv preprint arXiv:2301.08843}, 2023.

\bibitem[Maro{\~n}as \& Hern{\'a}ndez-Lobato(2022)Maro{\~n}as and
  Hern{\'a}ndez-Lobato]{maronas2022efficient}
Maro{\~n}as, J. and Hern{\'a}ndez-Lobato, D.
\newblock Efficient transformed gaussian processes for non-stationary dependent
  multi-class classification.
\newblock \emph{arXiv preprint arXiv:2205.15008}, 2022.

\bibitem[Matthews et~al.(2017)Matthews, Van Der~Wilk, Nickson, Fujii,
  Boukouvalas, Le{\'o}n-Villagr{\'a}, Ghahramani, and
  Hensman]{matthews2017gpflow}
Matthews, A. G. d.~G., Van Der~Wilk, M., Nickson, T., Fujii, K., Boukouvalas,
  A., Le{\'o}n-Villagr{\'a}, P., Ghahramani, Z., and Hensman, J.
\newblock Gpflow: A gaussian process library using tensorflow.
\newblock \emph{J. Mach. Learn. Res.}, 18\penalty0 (40):\penalty0 1--6, 2017.

\bibitem[Park et~al.(2021)Park, Kim, Lee, Choo, Lee, Kim, and
  Choi]{park2020vid}
Park, S., Kim, K., Lee, J., Choo, J., Lee, J., Kim, S., and Choi, E.
\newblock Vid-ode: Continuous-time video generation with neural ordinary
  differential equation.
\newblock \emph{arXiv preprint arXiv:2010.08188}, pp.\  online, 2021.

\bibitem[Pascanu et~al.(2013)Pascanu, Mikolov, and
  Bengio]{pascanu2013difficulty}
Pascanu, R., Mikolov, T., and Bengio, Y.
\newblock On the difficulty of training recurrent neural networks.
\newblock In \emph{International conference on machine learning}, pp.\
  1310--1318. PMLR, 2013.

\bibitem[Pearce(2020)]{pearce_gaussian_2020}
Pearce, M.
\newblock The gaussian process prior vae for interpretable latent dynamics from
  pixels.
\newblock In \emph{Symposium on {Advances} in {Approximate} {Bayesian}
  {Inference}}, pp.\  1--12. PMLR, 2020.

\bibitem[Ranganath et~al.(2014)Ranganath, Gerrish, and
  Blei]{ranganath2014black}
Ranganath, R., Gerrish, S., and Blei, D.
\newblock Black box variational inference.
\newblock In \emph{Artificial intelligence and statistics}, pp.\  814--822.
  PMLR, 2014.

\bibitem[Rasmussen(2003)]{rasmussen2003gaussian}
Rasmussen, C.~E.
\newblock Gaussian processes in machine learning.
\newblock In \emph{Summer school on machine learning}, pp.\  63--71. Springer,
  2003.

\bibitem[Ravuri et~al.(2021)Ravuri, Lenc, Willson, Kangin, Lam, Mirowski,
  Fitzsimons, Athanassiadou, Kashem, Madge, et~al.]{ravuri2021skilful}
Ravuri, S., Lenc, K., Willson, M., Kangin, D., Lam, R., Mirowski, P.,
  Fitzsimons, M., Athanassiadou, M., Kashem, S., Madge, S., et~al.
\newblock Skilful precipitation nowcasting using deep generative models of
  radar.
\newblock \emph{Nature}, 597\penalty0 (7878):\penalty0 672--677, 2021.

\bibitem[Rubanova et~al.(2019)Rubanova, Chen, and Duvenaud]{rubanova2019latent}
Rubanova, Y., Chen, R.~T., and Duvenaud, D.~K.
\newblock Latent ordinary differential equations for irregularly-sampled time
  series.
\newblock \emph{Advances in neural information processing systems}, 32, 2019.

\bibitem[S{\"a}rkk{\"a} \& Garc{\'\i}a-Fern{\'a}ndez(2020)S{\"a}rkk{\"a} and
  Garc{\'\i}a-Fern{\'a}ndez]{sarkka2020temporal}
S{\"a}rkk{\"a}, S. and Garc{\'\i}a-Fern{\'a}ndez, {\'A}.~F.
\newblock Temporal parallelization of bayesian smoothers.
\newblock \emph{IEEE Transactions on Automatic Control}, 66\penalty0
  (1):\penalty0 299--306, 2020.

\bibitem[S{\"a}rkk{\"a} \& Solin(2019)S{\"a}rkk{\"a} and
  Solin]{sarkka2019applied}
S{\"a}rkk{\"a}, S. and Solin, A.
\newblock \emph{Applied stochastic differential equations}, volume~10.
\newblock Cambridge University Press, 2019.

\bibitem[Schirmer et~al.(2022)Schirmer, Eltayeb, Lessmann, and
  Rudolph]{schirmer2022modeling}
Schirmer, M., Eltayeb, M., Lessmann, S., and Rudolph, M.
\newblock Modeling irregular time series with continuous recurrent units.
\newblock In \emph{International Conference on Machine Learning}, pp.\
  19388--19405. PMLR, 2022.

\bibitem[Sims(1980)]{sims1980macroeconomics}
Sims, C.~A.
\newblock Macroeconomics and reality.
\newblock \emph{Econometrica: journal of the Econometric Society}, pp.\  1--48,
  1980.

\bibitem[Solin(2016)]{solin_stochastic_2016}
Solin, A.
\newblock Stochastic differential equation methods for spatio-temporal
  {Gaussian} process regression.
\newblock 2016.
\newblock PhD Thesis: Aalto University.

\bibitem[Solin \& S{\"a}rkk{\"a}(2014)Solin and
  S{\"a}rkk{\"a}]{Solin2014ExplicitLB}
Solin, A. and S{\"a}rkk{\"a}, S.
\newblock Explicit link between periodic covariance functions and state space
  models.
\newblock In \emph{AISTATS}, 2014.

\bibitem[Taylor \& Letham(2018)Taylor and Letham]{taylor2018forecasting}
Taylor, S.~J. and Letham, B.
\newblock Forecasting at scale.
\newblock \emph{The American Statistician}, 72\penalty0 (1):\penalty0 37--45,
  2018.

\bibitem[Tebbutt et~al.(2021)Tebbutt, Solin, and
  Turner]{tebbutt_combining_2021}
Tebbutt, W., Solin, A., and Turner, R.~E.
\newblock Combining pseudo-point and state space approximations for
  sum-separable {Gaussian} processes.
\newblock In \emph{Uncertainty in {Artificial} {Intelligence}}, pp.\
  1607--1617. PMLR, 2021.

\bibitem[Titsias(2009)]{titsias2009variational}
Titsias, M.
\newblock Variational learning of inducing variables in sparse gaussian
  processes.
\newblock In \emph{Artificial intelligence and statistics}, pp.\  567--574.
  PMLR, 2009.

\bibitem[Tunyasuvunakool et~al.(2020)Tunyasuvunakool, Muldal, Doron, Liu,
  Bohez, Merel, Erez, Lillicrap, Heess, and Tassa]{tunyasuvunakool2020}
Tunyasuvunakool, S., Muldal, A., Doron, Y., Liu, S., Bohez, S., Merel, J.,
  Erez, T., Lillicrap, T., Heess, N., and Tassa, Y.
\newblock dmcontrol: Software and tasks for continuous control.
\newblock \emph{Software Impacts}, 6:\penalty0 100022, 2020.
\newblock ISSN 2665-9638.
\newblock \doi{https://doi.org/10.1016/j.simpa.2020.100022}.
\newblock URL
  \url{https://www.sciencedirect.com/science/article/pii/S2665963820300099}.

\bibitem[Wilkinson et~al.(2020)Wilkinson, Chang, Andersen, and
  Solin]{wilkinson_state_2020}
Wilkinson, W., Chang, P., Andersen, M., and Solin, A.
\newblock State {Space} {Expectation} {Propagation}: {Efficient} {Inference}
  {Schemes} for {Temporal} {Gaussian} {Processes}.
\newblock In III, H.~D. and Singh, A. (eds.), \emph{Proceedings of the 37th
  {International} {Conference} on {Machine} {Learning}}, volume 119 of
  \emph{Proceedings of {Machine} {Learning} {Research}}, pp.\  10270--10281.
  PMLR, July 2020.
\newblock URL \url{https://proceedings.mlr.press/v119/wilkinson20a.html}.

\bibitem[Wilkinson et~al.(2021)Wilkinson, Solin, and
  Adam]{wilkinson_sparse_2021}
Wilkinson, W.~J., Solin, A., and Adam, V.
\newblock Sparse {Algorithms} for {Markovian} {Gaussian} {Processes}.
\newblock \emph{arXiv:2103.10710 [cs, stat]}, June 2021.
\newblock URL \url{http://arxiv.org/abs/2103.10710}.
\newblock arXiv: 2103.10710.

\bibitem[Zaheer et~al.(2017)Zaheer, Kottur, Ravanbakhsh, Poczos, Salakhutdinov,
  and Smola]{zaheer2017deep}
Zaheer, M., Kottur, S., Ravanbakhsh, S., Poczos, B., Salakhutdinov, R.~R., and
  Smola, A.~J.
\newblock Deep sets.
\newblock \emph{Advances in neural information processing systems}, 30, 2017.

\bibitem[Øksendal(2003)]{oksendal_stochastic_2003}
Øksendal, B.
\newblock Stochastic differential equations.
\newblock In \emph{Stochastic differential equations}, pp.\  65--84. Springer,
  2003.

\end{thebibliography}
\bibliographystyle{icml2023}
}

\appendix
\onecolumn
\section{Markovian Gaussian Process Background}
\label{appendix:mgpvae}
In this section, we provide a rigorous treatment of the derivation of the state space Markovian Gaussian process equations as previously presented in \citet{sarkka2019applied} with stochastic differential equation technicalities from \citet{oksendal_stochastic_2003, evans_introduction_2006}.

\subsection{Stochastic Differential Equations}
\begin{definition}[Univariate Brownian Motion; \citep{oksendal_stochastic_2003,evans_introduction_2006}]
\label{defn:uni_brownian}
  Let $(\Omega, \mathcal{F},\mathbb{P})$ be a probability space. $B:[0,T]\times \Omega\rightarrow\mathbb{R}$ is a univariate Brownian motion if:
  \begin{enumerate}
    \item $B_0=0$, almost surely. 
    \item  $t\mapsto B(t,\cdot)$ is continuous almost surely. 
    \item  For $t_4>t_3\geq t_2>t_1$, $B_{t_4}-B_{t_3}\indep B_{t_2}-B_{t_1}$. 
    \item For any $\delta>0$, $B_{t+\delta} - B_t\sim\mathcal{N}(0, \delta q)$, where $q$ the correlation factor.
  \end{enumerate}
\end{definition}
In a similar manner, we may extend univariate Brownian motions into correlated multi-dimensional forms.
\begin{definition}[Multi-dimensional Brownian Motion]
  $B_t$ is an $e$-dimensional (correlated) Brownian motion:  each $B^k$, for $k=1,\ldots, e$, is a univariate Brownian motion. We assume that each $B^k$ is correlated, resulting in $B_{t+\delta}-B_t\sim\mathcal{N}(0, \delta \mathbf{Q}_c)$, where $\mathbf{Q}_c\in\mathbb{R}^{e\times e}$ is the spectral density matrix.
\end{definition}
Define $\mathcal{F}_t$ to be the $\sigma$-algebra generated by $B_s(\cdot)$ for $s\leq t$. Intuitively, this is the `history of information' up to time $t$ created by $B_s$ for $s\leq t$. Therefore for a stochastic process driven by $B_s$ up to time $t$, we would like it to be measurable at all times $t$, which motivates the next definition of measurability for stochastic processes:
\begin{definition}
  Let \{$\mathcal{F}_{t}\}_{t\geq 0}$ be an increasing family of $\sigma$-algebras generated by $B_t$ i.e. for $s<t$, $\mathcal{F}_s\subset \mathcal{F}_t$. Then $v(t,\omega):[0,\infty)\times\Omega\rightarrow\mathbb{R}^e$ is $\mathcal{F}_t$-adapted if $v(t,\omega)$ is measurable in $\mathcal{F}_t$ for all $t\geq 0$.
\end{definition}

With the correlated formulation, we can rederive It\^{o} isometry using the same proof as in the uncorrelated case \citep{oksendal_stochastic_2003, evans_introduction_2006}. We use the definition of the multi-dimensional It\^{o} integral in Definition 3.3.1~\citet{oksendal_stochastic_2003}. Let $\mathbb{E}$ be the expectation with respect to $\mathbb{P}$.
\begin{definition}[Multi-dimensional It\^{o} Integral \citep{oksendal_stochastic_2003}]
  Let $\mathcal{V}^{d\times e}(S,T)$, for $d\in\mathbb{N}$ and $S<T$, be a set of matrix-valued stochastic processes $v(t,\omega)\in\mathbb{R}^{d\times e}$ where each entry $v_{ij}(t, \omega)$ satisfies:
  \begin{itemize}
      \item $(t,\omega)\rightarrow v_{ij}(t,\omega)$ is $\mathcal{B}(\mathbb{R})\times\mathcal{F}$-measurable, where $\mathcal{B}(\mathbb{R})$ is the Borel $\sigma$-algebra on $[0,\infty)$.
      \item $v_{ij}(\tau,\omega)\in L^2([S,T])$ in expectation i.e. $\mathbb{E}[\int_S^T v_{ij}^2(\tau,\cdot) \text{d}\tau]$.
      \item Let $W_t$ be a univariate Brownian motion. There exists an increasing family of $\sigma$-algebras \{$\mathcal{H}_t\}_{t\geq 0}$ such that (1) $W_t$ is a martingale with respect to $\mathcal{H}_t$ and (2) $v_{ij}(t,\omega)$ is $\mathcal{H}_t$-adapted.
  \end{itemize}
  Then, we can define the multi-dimensional It\^{o} integral as follows: For all $v\in\mathcal{V}^{d\times e}(S,T)$,
  \begin{align*}
    \int_S^T v(\tau,\omega) \text{d}B_\tau = \int_S^T \begin{pmatrix}
    v_{11}(\tau,\omega) &\cdots & v_{1e}(\tau,\omega) \\
    \vdots & \ddots & \vdots \\
    v_{d1}(\tau,\omega) &\cdots & v_{de}(\tau,\omega) \\
    \end{pmatrix}
\begin{pmatrix}
    \text{d}B^1_\tau \\
    \vdots \\
    \text{d}B^e_\tau\\
    \end{pmatrix},
  \end{align*}
  where $\bigg[\int_S^T v(\tau,\omega) \text{d}B_\tau\bigg]_{i}=\sum_{j=1}^e \int_S^T v_{ij}(\tau,\omega) dB^j_\tau$.
\end{definition}

\begin{proposition}[Multi-dimensional It\^{o} Isometry]
\label{prop:Ito_isometry}
  For $\mathbf{F}, \mathbf{G}\in \mathcal{V}^{d\times e}$ (as defined in \citet{oksendal_stochastic_2003}) and $S<T$, then 
  \begin{talign*}
    \mathbb{E}[\int_{S}^{T} \mathbf{F}(\tau, \cdot) \text{d}B_\tau][\int_S^{T} \mathbf{G}(\tau, \cdot) \text{d}B_\tau]^\intercal = \mathbb{E}\int_S^T \mathbf{F}(\tau,\cdot) \mathbf{Q}_c \mathbf{G}(\tau,\cdot)^\intercal \text{d}\tau.
  \end{talign*}
  Furthermore, if $\bF$ and $\mathbf{G}$ are deterministic, then for $t_1,t_2\geq t_0$,
    \begin{talign*}
    \mathbb{E}[\int_{t_0}^{t_1} \mathbf{F}(\tau) \text{d}B_\tau][\int_{t_0}^{t_2} \mathbf{G}(\tau) \text{d}B_\tau]^\intercal = \mathbb{E}\int_{t_0}^{\min(t_1,t_2)} \mathbf{F}(\tau) \mathbf{Q}_c \mathbf{G}(\tau)^\intercal \text{d}\tau.
  \end{talign*}
\end{proposition}

\begin{proof}
Let $\{e_i\}_{i=1}^e$ be the canonical basis in $\mathbb{R}^e$. We first show that $ik$-th component of $\bigg[\mathbb{E}[\int_{S}^{T} \mathbf{F}(\tau, \cdot) \text{d}B_\tau][\int_S^{T} \mathbf{G}(\tau, \cdot) \text{d}B_\tau]^\intercal\bigg]\in\mathbb{R}^{d\times d}$, for $S<T$, is equal to 
\begin{talign*}
  \int_S^T \bF_i^\intercal(\tau,\cdot) \bQ_c [\mathbf{G}(\tau,\cdot)]_{k}^{\intercal} \text{d}\tau&=\mathbb{E}\int_S^T e_i^\intercal\bF^\intercal(\tau,\cdot) \bQ_c [\mathbf{G}(\tau,\cdot)]^\intercal_k \text{d}\tau\\
  &=e_i^\intercal\bigg[\mathbb{E}\int_S^T \bF^\intercal(\tau,\cdot) \bQ_c \mathbf{G}(\tau,\cdot)^\intercal \text{d}\tau\bigg] e_k\\
  &= \bigg[\mathbb{E}\int_S^T \bF^\intercal(\tau,\cdot) \bQ_c \mathbf{G}(\tau,\cdot)^\intercal \text{d}\tau\bigg]_{ik},
\end{talign*}
which would complete the proof. Indeed, the $ik$-th component is equal to
\begin{talign*}
  \mathbb{E}\bigg[\sum_{j=1}^e \int_S^T \bF_{ij}(\tau,\cdot) \text{d}B_\tau^j\bigg]\bigg[\sum_{l=1}^e \int_S^T \mathbf{G}_{kl}(\tau,\cdot) \text{d}B_\tau^l\bigg] &= \bigg[\sum_{j,l=1}^e \mathbb{E}\int_S^T \bF_{ij}(\tau,\cdot) \text{d}B_\tau^j \int_S^T \mathbf{G}_{kl}(\tau,\cdot) \text{d}B_\tau^l\bigg] \\
  &\stackrel{\mathclap{\text{\tiny 1D It\^{o} Isometry}}}{=} \qquad\bigg[\sum_{j,l=1}^e \mathbb{E}\int_S^T \bF_{ij}(\tau,\cdot)[\bQ_c]_{jl}\mathbf{G}_{kl}(\tau,\cdot) \text{d}\tau\bigg] \\
  &=\sum_{j,l=1}^e \mathbb{E}\int_S^T \bF_{ij}(\tau,\cdot)[\bQ_c]_{jl}[\mathbf{G}^\intercal(\tau,\cdot)]_{lk} \text{d}\tau\\
    &=\mathbb{E}\int_S^T e_i^\intercal \bF(\tau,\cdot)\bQ_c[\mathbf{G}^\intercal(\tau,\cdot)]e_k \text{d}\tau\\
    &=\bigg[\mathbb{E}\int_S^T  \bF(\tau,\cdot)\bQ_c[\mathbf{G}^\intercal(\tau,\cdot)] \text{d}\tau\bigg]_{ik},
\end{talign*}
as required. Next, if $\mathbf{F}$ and $\mathbf{G}$ are deterministic and suppose that $t_1<t_2$ without loss of generality, then
\begin{talign*}
    \mathbb{E}[\int_{t_0}^{t_1} \mathbf{F}(\tau) \text{d}B_\tau][\int_{t_0}^{t_2} \mathbf{G}(\tau) \text{d}B_\tau]^\intercal &= \mathbb{E}[\int_{t_0}^{t_1} \mathbf{F}(\tau) \text{d}B_\tau][\int_{t_0}^{t_1} \mathbf{G}(\tau) \text{d}B_\tau + \cancel{\int_{t_1}^{t_2} \mathbf{G}(\tau) \text{d}B_\tau}]^\intercal \\
    &=\mathbb{E}[\int_{t_0}^{\min(t_1,t_2)} \mathbf{F}(\tau) \text{d}B_\tau][\int_{t_0}^{\min(t_1,t_2)} \mathbf{G}(\tau) \text{d}B_\tau]^\intercal \\
    &=\mathbb{E}\int_{t_0}^{\min(t_1,t_2)} \mathbf{F}(\tau)\bQ_c\mathbf{G}^\intercal(\tau) \text{d}B_\tau,
\end{talign*}
where for the first line we used property (iii) of Definition~\ref{defn:uni_brownian} of Brownian motions in multi-dimensions.
\end{proof}

\subsection{Markovian Gaussian Processes}
A Markovian Gaussian process $f\sim\mathcal{GP}(0,k)$ can be written with an SDE of latent dimension $d$ 
\begin{talign}
  \text{d}\mathbf{s}(t) = \mathbf{F}\mathbf{s}(t) \text{d}t + \mathbf{L}\text{d}B_t,\quad f(x)=\mathbf{H}\mathbf{s}(t),
\end{talign}
where $\mathbf{F}\in\mathbb{R}^{d\times d}, \mathbf{L}\in\mathbb{R}^{d\times e}, \mathbf{H}\in\mathbb{R}^{1\times d}$ are the feedback, noise effect and emission matrices, and $B_t$ is an $e$-dimensional (correlated) Brownian motion with spectral density matrix $\mathbf{Q}_c$. Suppose that $\bs(t_0)\sim\mathcal{N}(\mathbf{m}_0, \mathbf{P}_0)$ with the stationary state mean and covariance. Note that we assume that $s(t_0)$ independent of $\mathcal{F}_{t_0}^+$, the $\sigma$-algebra generated by $B_t-B_s$ for $t\geq s\geq t_0$, in order to invoke existence and uniquess of SDE solutions (Existence and Uniqueness Theorem, page 90, \citet{evans_introduction_2006}). Thus the linear SDE equation~\ref{eqn:temporal_sde} admits a unique closed form solution:
\begin{talign*}
  \mathbf{s}(t) = e^{(t-t_0)\mathbf{F}}\mathbf{s}(t_0) + \int_{t_0}^t e^{(t-\tau)\mathbf{F}}\mathbf{L}\text{d}B_\tau. 
\end{talign*}
We have
\begin{talign*}
  m(t)=\mathbb{E}[\mathbf{s}(t)] = e^{(t-t_0)\mathbf{F}}\mathbf{m}_0,
\end{talign*}
since for any $f\in \mathcal{V}^{d\times e}$, $\mathbb{E}[\int_S^T f(\tau) \text{d}B_\tau]=0$ \citep{oksendal_stochastic_2003,evans_introduction_2006}. To calculate the covariance, we have
\begin{talign*}
  \kappa(t,t') &= \mathbb{E}[\mathbf{s}(t)-m(t))(\mathbf{s}(t')-m(t'))^\intercal]\\
  &=e^{(t-t_0)\mathbf{F}}\mathbb{E}[(\mathbf{s}(t_0)-\mathbf{m}_0)(\mathbf{s}(t_0) - \mathbf{m}_0)^\intercal][e^{(t'-t_0)\mathbf{F}}]^\intercal \\
  &+ \mathbb{E}[\int_{t_0}^t e^{(t-\tau)\mathbf{F}}\text{d}B_\tau][\int_{t_0}^{t'} e^{(t'-\tau)\mathbf{F}}\text{d}B_{\tau'}]^\intercal \\
  &+ \cancel{\mathbb{E}\bigg[e^{(t-t_0)\mathbf{F}}\mathbf{L}(\mathbf{s}(t_0)-\mathbf{m}_0)(\int_{t_0}^{t'} e^{(t'-\tau)\mathbf{F}}\mathbf{L}\text{d}B_{\tau})^\intercal\bigg]} \quad \mbox{(Independence)}\\
  &+\cancel{\mathbb{E}\bigg[(\int_{t_0}^{t} e^{(t-\tau)\mathbf{F}}\mathbf{L}\text{d}B_{\tau})(\mathbf{s}(t_0)-\mathbf{m}_0)^\intercal \mathbf{L}^\intercal [e^{(t-t_0)\mathbf{F}}]^\intercal\bigg]} \quad \mbox{(Independence)}\\
  &= e^{(t-t_0)\mathbf{F}}\mathbf{P}_0 [e^{(t'-t_0)\mathbf{F}}]^\intercal
  +\int_{t_0}^{\min(t,t')}e^{(t-\tau)\mathbf{F}}\mathbf{L}\mathbf{Q}_c\mathbf{L}^\intercal  [e^{(t'-\tau)\mathbf{F}}]^\intercal \text{d}\tau \quad \mbox{(It\^{o} Isometry; Proposition~\ref{prop:Ito_isometry})},
\end{talign*}
where we note that $\int_{t_0}^{t} e^{(t-\tau)\mathbf{F}}\mathbf{L}\text{d}B_{\tau}$ and $\int_{t_0}^{t'} e^{(t'-\tau)\mathbf{F}}\mathbf{L}\text{d}B_{\tau}$ are $\mathcal{F}_t$-measurable for $\mathcal{F}_t\subset\mathcal{F}_{t_0}^+$, and thus independent to $\bs(t_0)$.

Therefore for all $t\in\{t_0,\ldots,t_T\}$, if we solve for each $t_i$ with initial time $t_{i-1}$, then $s(t)\sim \mathcal{N}(m(t), k(t,t))$ with
\begin{talign*}
  m(t_{i+1})&= e^{(t_{i+1}-t_i)\bF} \bm_0\\
  k(t_{i+1},t_{i+1})&= e^{(t_{i+1}-t_i)\mathbf{F}}\mathbf{P}_0 [e^{(t_{i+1}-t_i)\mathbf{F}}]^\intercal
  +\int_{t_i}^{t_{i+1}}e^{(t_{i+1}-\tau)\mathbf{F}}\mathbf{L}\mathbf{Q}_c\mathbf{L}^\intercal  [e^{(t_{i+1}-\tau)\mathbf{F}}]^\intercal \text{d}\tau.
\end{talign*}
Denote $\Delta_{t_{i+1}}=t_{i+1}-t_i$, $\mathbf{A}_{i,i+1}=e^{(t_{i+1}-t_i)\mathbf{F}}=e^{\Delta_{t_{i+1}}\bF}$ and $\mathbf{Q}_{i,i+1}=\int_{t_0}^{\Delta_{t_{i+1}}+t_0}e^{(\Delta_{t_{i+1}}+t_0-\tau)\mathbf{F}}\mathbf{L}\mathbf{Q}_c\mathbf{L}^\intercal  [e^{(\Delta_{t_{i+1}}+t_0-\tau)\mathbf{F}}]^\intercal \text{d}\tau$. Then we can derive the recursive equations
\begin{talign*}
  \bs(t_{i+1})&=\mathbf{A}_{i,i+1}\bs(t_i)+\int_{t_i}^{t_{i+1}}e^{(t_{i+1}-\tau)\bF}\mathbf{L} \text{d}B_\tau\\
  &=\mathbf{A}_{i,i+1}\bs(t_i) + \int_{t_0}^{\Delta_{t_{i+1}}+t_0} e^{\Delta_{t_{i+1}}+t_0-\tau} \mathbf{L} \text{d}B_\tau\\ 
  &=\mathbf{A}_{i,i+1}\bs(t_i) + \mathbf{q}_{i},
\end{talign*}
where $\mathbf{q}_{i}\sim \mathcal{N}(0,\mathbf{Q}_{i,i+1})$.

\paragraph{Types of Kernels:} A large number of kernels do allow for the Markovian property to be satisfied. A full list of them could be found in \citet{solin_stochastic_2016, sarkka2019applied}. Here, we list 2 common kernels:

\begin{itemize}
    \item Matern-$\sfrac{3}{2}$: The kernel is $k(t,t')=\sigma^2(1+\frac{\sqrt{3}d}{\rho})\exp(-\frac{\sqrt{3}d}{\rho})$, where $d=|t-t'|$. With $\lambda=\frac{\sqrt{3}}{\ell}$:
    \begin{talign*}
        \bF = \begin{pmatrix}
            0&1 \\
            -\lambda^2 & -2\lambda
        \end{pmatrix},\quad \bm_0=\mathbf{0},\quad \bP_0 = \begin{pmatrix}
            \sigma^2&0 \\
            0 & \sigma^2\lambda^2
        \end{pmatrix},\quad \bQ_c=\frac{12\sqrt{3}\sigma^2}{\ell^2 },\quad \mathbf{L}=\begin{pmatrix}
            0\\
            1
        \end{pmatrix},\quad \bH=\begin{pmatrix}
            1\\
            0
        \end{pmatrix}.
    \end{talign*}

    \item Matern-$\sfrac{5}{2}$: The kernel is $k(t,t')=\sigma^2(1+\frac{\sqrt{5}d}{\rho} + \frac{5d^2}{3\rho^2})\exp(-\frac{\sqrt{5}d}{\rho})$, where $d=|t-t'|$.  With $\lambda=\frac{\sqrt{5}}{\ell}$ and $\kappa=\frac{5 \sigma^2}{3\ell^2}$:    \begin{talign*}
        \bF = \begin{pmatrix}
            0&1 & 0\\
            0 & 0 & 1\\
            -\lambda^3 & -3\lambda^2 & -3\lambda
        \end{pmatrix},\quad \bm_0=\mathbf{0},\quad \bP_0 = \begin{pmatrix}
            \sigma^2&0 &\kappa \\
            0 &  \kappa & 0\\
            -\kappa & 0 & \frac{25\sigma^2}{\ell^4}
        \end{pmatrix},\quad \bQ_c=\frac{400 \sqrt{5}\sigma^2}{3\ell^5},\quad \mathbf{L}=\begin{pmatrix}
            0\\
            0\\
            1
        \end{pmatrix},\quad \bH=\begin{pmatrix}
            1\\
            0\\
            0
        \end{pmatrix}.
    \end{talign*}
\end{itemize}

\paragraph{Kalman Filtering and Smoothing:} The full Kalman filtering and smoothing algorithm is delineated in Algorithm~\ref{alg:filtering_smoothing}.
\begin{algorithm}[tbp]
\caption{Kalman Filtering and Smoothing}
\label{alg:filtering_smoothing}
\begin{algorithmic}[1]
\STATE {\bfseries Inputs:} Variational parameters $\tilde{\bY}_{1:T},\tilde{\bV}_{1:T}$. Initial conditions $\mathbf{m}_0^f=\mathbf{m}_0$, $\mathbf{P}_0^f=\mathbf{P}_0$, transition matrices $\{\mathbf{A}_{i-1,i}, \mathbf{Q}_{i-1,i}\}_{i=1}^T$ and emission matrix $\mathbf{H}$.
\STATE \textbf{Filtering}:
\FOR{$i=1,\ldots,T$}
\STATE Compute predictive filter distribution $p(\bs_t|\tilde{\bY}_{1:i-1})=N(\bs_t|\bm^p_t,\bP^p_t)$:
\begin{talign*}
  \bm^p_{t_i} &= \bA_{i-1,i} \bm_{t_{i-1}}^f \\
  \bP^p_{t_i} &= \bA_{i-1,i} \bP_{t_{i-1}}^f \bA_{i-1,i}^\intercal + \bQ_{i-1,i}.
\end{talign*}
\STATE Let $\Lambda_{t_i}=\bH \bP^p_{t_i} \bH^\intercal+\tilde{\bV}_{t_i}$. Compute log marginal likelihood:
\begin{talign*}
  \ell_{t_i} = \log\mathbb{E}_{p(\bs_{t_i}|\tilde{\bY}_{1:i-1})} N(\tilde{\bY}_{t_i}; \mathbf{H}\bs_{t_i}, \tilde{\mathbf{V}}_{t_i})= \log N(\tilde{\bY}_{t_i}; \mathbf{H}\bm^p_{t_i}, \Lambda_{t_i})
\end{talign*}
\STATE Compute updated filter distribution $p(\bs_{t_i}|\tilde{\bY}_{1:i})=N(\bs_{t_i}|\bm^f_{t_i},\bP^f_{t_i})$:
\begin{talign*}
  \mathbf{W}_{t_i} = \bP_{t_i}^p \bH^\intercal \Lambda_{t_i}^{-1}\Rightarrow \bm_{t_i}^f =\bm_{t_i}^p + \mathbf{W}_{t_i}(\tilde{\bY}_{t_i}-\bH\bm_{t_i}^p) ,\quad \bP_{t_i}^f= \bP_{t_i}^p - \mathbf{W}_{t_i}\Lambda_{t_i}\mathbf{W}_{t_i}^\intercal.
\end{talign*}
\ENDFOR
\STATE \textbf{Smoothing}, with initial conditions $\bm_T^s=\bm_T^f$, $\bP_T^s=\bm_T^f$:
\FOR{$i=T-1,\ldots, 1$}
\STATE Compute the smoothing distribution $q(\bs_{t_i})=p(\bs_{t_i}|\tilde{\bY}_{1:T})=N(\bs_{t_i}|\bm^s_{t_i},\bP^s_{t_i})$ with the RTS smoother:
\begin{talign*}
  \mathbf{G}_{t_i} = \bP_{t_i}^{f}\bA_{i,i+1}[\bP_{t_{i+1}}^p]^{-1}\Rightarrow \bm_{t_i}^s = \bm_{t_i}^f + \mathbf{G}_{t_i}(\bm_{t_{i+1}}^s - \bm_{t_{i+1}}^p),\quad \bP_{t_i}^s = \bP_{t_i}^f + \mathbf{G}_{t_i}(\bP_{t_{i+1}}^{s} - \bP_{t_{i+1}}^p)\mathbf{G}_{t_i}^\intercal
\end{talign*}
\ENDFOR
\vspace{1mm}
\STATE \textbf{Return:} log marginal likelihood $\sum_{t=1}^T \ell_{t_i}$, marginal posteriors $\{q(\bs_{t_i})\}_{i=1}^{T}$.
\end{algorithmic}
\end{algorithm}

\subsection{Further details on spatiotemporal MGPVAE}
\label{appendix:spatiotemporal}
To perform prediction at arbitrary spatiotemporal locations $(r_*,t_*)$, we can use a conditional independence property proven in \citet{tebbutt_combining_2021} for separable spatiotemporal Markovian GPs that yields the predictive distribution
\begin{talign}
\label{eqn:spatial_marginal}
&p(\bZ(r_*, t_*)|\bY)\approx q(\bZ(r_*, t_*)) \\
:=&\int p(\bZ(r_*, t_*)|\bs(\mathbf{R}, t_*))
q(\bs(\mathbf{R},t_*)|\bs(\mathbf{R}, t_{1:T}))\text{d}\bs(\mathbf{R}, t_*), \nonumber
\end{talign}
where $q(\bs(\mathbf{R}, t_{1:T}))=N(\bs(\mathbf{R}, t_{1:T})|m(\mathbf{R}, t_{1:T}), c(\mathbf{R}, t_{1:T}))$ is the approximate state posterior $s(\mathbf{R}, t_*)$ conditioned on $\mathbf{s}(\mathbf{R}, t_{1:T})$. Given in \citet{tebbutt_combining_2021} and \citet{wilkinson_sparse_2021},
\begin{talign*}
  p(\bZ(r_*, t_*)|\bs(\mathbf{R}, t_*))=&N(\bZ(r_*,t_*)|\mathbf{B}_{r_*}\bs(\mathbf{R},t_*), \mathbf{C}_{r_*}),
\end{talign*}
where 
\begin{talign*}
  \mathbf{B}_{r_*}&= [\mathbf{K}^r_{r_*\mathbf{R}}(\mathbf{K}^r_{\mathbf{R}\mathbf{R}})^{-1}]\otimes[\mathbf{L}^r_{\mathbf{R}\mathbf{R}}\otimes \bH^t]\\
  \mathbf{C}_{r_*}&=k_{t}(0,0)[\mathbf{K}_{r_*r_*}^r - \mathbf{K}^r_{r_*\mathbf{R}}(\mathbf{K}^r_{\mathbf{R}\mathbf{R}})^{-1}\mathbf{K}^r_{\mathbf{R}r_*}].
\end{talign*}
Therefore the integral in Equation~(\ref{eqn:spatial_marginal}) yields $q(\bZ(r_*, t_*))=N(\bZ(r_*, t_*)|\mathbf{B}_{r_*}m(t_{1:T}, \mathbb{R}), \mathbf{B}_{r_*}c(t_{1:T, \mathbb{R}})\mathbf{B}_{r_*}^\intercal + \mathbf{C}_{r_*})$.

\section{Additional Experimental and Implementation Details}
\label{appendix:experiments}
For the generating modelling tasks of corrupt image (re)-generation, the NLL on an test sequence $\bY=(\bY_1,\ldots,\bY_{100})$ is
\begin{talign}
\label{eqn:nll_seen}
  \log p(\bY)&=\log \int p(\bY|\bZ) p(\bZ) \text{d}\bZ \nonumber\\
  &= \log \int p(\bY|\bZ) \frac{p(\bZ)}{q(\bZ|\bY_{\text{corrupt}})} q(\bZ|\bY_{\text{corrupt}})\text{d}\bZ \nonumber\\
  &=\log \frac{1}{K}\sum_{k=1}^K p(\bY|\bZ_k) \frac{p(\bZ_k)}{q(\bZ_k|\bY_{\text{corrupt}})},\quad \bZ_k\sim q(\bZ|\bY_{\text{corrupt}}),\quad k=1\ldots,K \nonumber\\
  &= \log\frac{1}{K} + \text{logsumexp}_{k=1,\ldots,K}\bigg[ \log p(\bY|\bZ_k) - \log\frac{q(\bZ_k|\bY)}{p(\bZ_k)} \bigg] .
\end{talign}
For the tasks of missing frame imputation, given test sequence $\bY=(\bY_1,\ldots,\bY_{100})$, write $\bY\equiv(\bY,\bY^c)$, where $\bY^c$ are the unseen frames and $\bY^c$ are seen frames.
\begin{talign*}
  \log p(\bY)&\equiv\log p(\bY^c | \bY) + \log p(\bY).
\end{talign*}
$\log p(\bY)$ can be computed by Equation~\ref{eqn:nll_seen} and 
\begin{talign*}
\log p(\bY^c| \bY) &= \log\int p(\bY^c|\bZ^c) p(\bZ^c|\bY)  \text{d}\bZ^c \\
&=\log\int p(\bY^c|\bZ^c) \bigg[ \int p(\bZ^c|\bZ)p(\bZ|\bY)\text{d}\bZ \bigg] \text{d}\bZ^c \\
&=\log\int p(\bY^c|\bZ^c) \bigg[ \underbrace{\int p(\bZ^c|\bZ)\underbrace{q(\bZ|\bY)}_{\text{encoder}}\text{d}\bZ}_{=:q(\bZ^c|\bY)\approx p(\bZ^c|\bY)} \bigg] \text{d}\bZ^c \\
&\approx \log\int p(\bY^c|\bZ^c) q(\bZ^c|\bY) \text{d}\bZ^c \\
&= \log \frac{1}{K}\sum_{k=1}^K p(\bY^c|\bZ^c_k), \quad \bZ_k^c\sim q(\bZ^c|\bY),\quad k=1\ldots,K,
\\
&= \log \frac{1}{K} + \text{logsumexp}_{k=1,\ldots,K} \log p(\bY^c|\bZ^c_k),
\end{talign*}
$q(\bZ^c|\bY)$ can be analytically obtained (e.g. GP posterior distribution). For VRNN/KVAE, $q(\bZ^c|\bY)$ would just be determined by the outputs of the RNN at each time step $t=1,\ldots,100$.

For the spatiotemporal modelling task, we compute the negative log predictive distribution (NLPD) at new spatial locations $s^*\notin\mathcal{S}$ and temporal locations $t\in\mathcal{T}$ observed during training, conditioned on observed data $\mathcal{D}=\{\bY(s,t)\}_{s\in \mathcal{S}, t\in\mathcal{T}}$, which is 
\begin{talign*}
\log p(\bY(t,s_*)|\mathcal{D}) &= \log \int p(\bY(t,s_*) | \bZ(t,s_*)) p(\bZ(t,s_*) | \mathcal{D}) \text{d}\bZ(t,s_*) \\ 
&\approx \log \int p(\bY(t,s_*) | \bZ(t,s_*)) q(\bZ(t,s_*) | \mathcal{D}) \text{d}\bZ(t,s_*) \\
&= \log \frac{1}{K}\sum_{k=1}^K  p(\bY(t,s_*)|\bZ_k(t,s_*)), \quad \bZ_k(t,s_*)\sim q(\bZ(t,s_*)|\mathcal{D}),\quad  k=1,\ldots,K, \\
&= \log\frac{1}{K}+\text{logsumexp}_{k=1,\ldots,K} \log p(\bY(t,s_*)|\bZ_k(t,s_*)).
\end{talign*}

We hereby provide a derivation of Equation \ref{eqn:nll_seen} for each model (using the original notation as much as possible).

\paragraph{MGPVAE:}
\begin{talign*}
\log p(\bY) &= \log \int p(\bY|\bs) \frac{p(\bs)}{q(\bs)} q(\bs)  \text{d}\bs,\\
&= \log \int p(\bY|\bs) \frac{p(\bs)  \int p(\tilde{\bY}|\bH\bs, \tilde{\bV}) p(\bs)\text{d}\bs}{p(\tilde{\bY}|\bH\bs, \tilde{\bV})p(\bs)} q(\bs)  \text{d}\bs,\\
&\approx \log\frac{1}{K} + \log \sum_{k=1}^K p(\bY|\bs) \frac{\int p(\tilde{\bY}|\bH\bs_k, \tilde{\bV}) p(\bs_k)\text{d}\bs}{p(\tilde{\bY}|\bH\bs_k, \tilde{\bV})},\quad \bs_k\sim q(\bs_k) \\
&= \log \frac{1}{K} + \text{logsumexp}_{k=1,\ldots,K} \log p(\bY|\bs_k) - \log p(\tilde{\bY}|\bH\bs_k,\tilde{\bV}) + \log\mathbb{E}_{p(\bs)} N(\tilde{\bY}| \mathbf{H}\bs, \tilde{\bV}) 
\end{talign*}
\paragraph{KVAE:}
When encountering a missing frame, the Kalman gain is zero in the Kalman filtering and smoothing operations. The filter prediction distribution is then used to recompute the $A$ and $C$ matrices, which are finally fed back into the usual filtering algorithm.

\begin{talign*}
    \log p(\bY) &= \log \int \frac{p(\bY|\ba) p(\ba|\bZ) p(\bZ) q(\ba,\bZ|\bY)}{q(\ba,\bZ|\bY)} \text{d}\ba \text{d}\bZ,\quad  \bigg[q(\ba,\bZ|\bY)=q(\ba|\bY)p(\bZ|\ba) \bigg]\\
    &\approx \log \sum_{k=1,\ldots K} \frac{p(\bY|\ba_k)p(\ba_k|\bZ_k)p(\bZ_k)}{q(\ba_k|\bY_k)p(\bZ_k|\bY_k)} + \log\frac{1}{K},\quad (\ba_k,\bZ_k)\sim q(\ba_k,\bZ_k|\bY) \\
    &= \log\frac{1}{K} + \text{logsumexp}_{k=1,\ldots,K} \log p(\bY|\ba_k) - \log q(\ba_k|\bY) + \log p(\ba_k|\bZ_k) + \log p(\bZ_k) - \log p(\bZ_k|\ba_k).
\end{talign*}

\paragraph{VRNN:}
When encountering a missing frame, the previous hidden state is fed into prior network and $\bZ_t$ is sampled from the prior. It is then decoded and the decoded output is then used as a pseudo-datapoint for autoencoding. 

\begin{talign*}
\log p(\bY) &= \log \prod_{t=1}^T p(\bY_t|\bY_{<t})\\
&= \log \int \prod_{t=1}^T p(\bY_t|\bY_{<t}, \bZ_{\leq t}) \frac{p(\bZ_t|\bY_{<t}, \bZ_{<t})}{q(\bZ_t|\bY_{\leq t}, \bZ_{<t})} q(\bZ_t|\bY_{\leq t}, \bZ_{<t}) \text{d}\bZ,\\
&\approx \log \frac{1}{K} + \log \sum_{k=1}^K  \prod_{t=1}^T p(\bY_t|\bY_{<t}, \bZ^k_{\leq t}) \frac{p(\bZ_t^k|\bY_{<t} \bZ^k_{<t})}{q(\bZ^k_t|\bY_{\leq t}, \bZ^k_{<t})},\quad \bZ_{1:T}^k\sim q(\bZ_{\leq T}|\bY_{\leq T}) = \prod_{t=1}^T q(\bZ_t|\bY_{\leq t}, \bZ_{<t}) \\
&= \log \frac{1}{K} + \text{logsumexp}_{k=1,\ldots,K} \sum_{t=1}^T \log p(\bY_t|\bY_{<t}, \bZ^k_{\leq t}) -  \log \frac{q(\bZ_t^k|\bY_{\leq t} \bZ^k_{<t})}{p(\bZ^k_t|\bY_{< t}, \bZ^k_{<t})}
\end{talign*}

\paragraph{LatentODE:}
The latentODE places a prior over the initial conditions $\bZ_0\sim p(\bZ_0)$ and a posterior via ODERNN $q(\bZ_0|\bY_{1:T})\equiv q(\bZ_0)$. Thus the NLL is
\begin{talign*}
  \log p(\bY) &= \log \int p(\bY|\bZ_0) \frac{p(\bZ_0)}{q(\bZ_0)} q(\bZ_0) d\bZ \\
  &\approx \log \frac{1}{K} + \text{logsumexp}_{k=1,\ldots,K} \log p(\bY|\bZ_0^k) - \log \frac{q(\bZ_0^k)}{p(\bZ_0^k)},\quad \bZ_0^k\sim p(\bZ_0).
\end{talign*}
\paragraph{GPVAE:}

\begin{talign*}
  \log p(\bY) &= \log \int p(\bY|\bZ) \frac{p(\bZ)}{q(\bZ)} q(\bZ) d\bZ \\
  &\approx \log \frac{1}{K} + \text{logsumexp}_{k=1,\ldots,K} \log p(\bY|\bZ_k) - \log\frac{q(\bZ_k|\bY)}{p(\bZ_k)},\quad \bZ_k\sim q(\bZ_k|\bY).
\end{talign*}

\paragraph{SVGPVAE:} We are given that $q(\bZ,\bG_m)=p(\bZ|\bG_m)q(\bF_m|\mu,\bA)$, where $\bG_m$ is the inducing variable and $(\mu, \bA)$ are the inducing variable mean and covariance. 
\begin{talign*}
\log p(\bY) &= \log \int p(\bZ,\bG_m,\bY) \frac{q(\bZ,\bG_m)}{q(\bZ,\bG_m)} \text{d}\bZ \text{d}\bG_m \\
&=\log \int \frac{p(\bY|\bZ)p(\bZ|\bG_m)p(\bG_m)}{p(\bZ|\bG_m)q(\bG_m)} p(\bZ|\bG_m)q(\bG_m) \text{d}\bZ \text{d}\bG_m\\
&=\log \int \frac{p(\bY|\bZ)p(\bG_m)}{q(\bG_m)} p(\bZ|\bG_m)q(\bG_m) \text{d}\bZ \text{d}\bG_m\\
&\approx\frac{1}{P}\sum_{p=1}^P \log\frac{1}{K} + \text{logsumexp}_{k=1,\ldots,K} \log p(\bY|\bZ_k^p) - \log \frac{q(\bG_m^k)}{p(\bG_m^k)},\quad \bG_m^k\sim q(\bG_m)\quad \bZ_k^p\sim p(\bZ_k|\bG_m^k),
\end{talign*}
where $p(\bZ_k|\bG_m)$ is a standard GP conditional distribution:
\begin{align*}
    p(\bZ_k|\bG_m)&\sim N\bigg(\bZ_k | \bK_{xm}\bK_{mm}^{-1}\bG_m, \bK_{xx} - \bK_{xm}\bK_{mm}^{-1}\bK_{mx} \bigg), \\
    q(\bG_m)\equiv q(\bG_m|\mu, \bA)&\sim N\bigg(\bG_m | \mu, \bA \bigg).
\end{align*}

\subsection{Video Data}
For each of the 4000 train and 1000 test MNIST images, we create a $T=100$ length sequence of rotating MNIST images by applying a clockwise rotation with period $t=50$, meaning that the image gets rotated twice. For corrupt video sequences, we randomly remove 60$\%$ of the pixels and replace them with the value 0. For the missing frames imputation task, we randomly remove 60$\%$ of the frames and keep track of the time steps for which the frames are removed.

For each model during training, we used the following configurations (note that we use the standard PyTorch-TensorFlow-JAX notation for network layers).
\begin{itemize}
  \item Batch size: 40
  \item Training epochs: 300
  \item Number of latent channels: $L=16$
  \item Adam optimizer learning rate: 1e-3
  \item clipGradNorm(model parameters, 100)
  \item Encoder structure: Conv(out=32, k=3, strides=2), ReLU(), Conv2D(out=32, k=3, strides=2), ReLU(), Flatten(), hiddenToVariationalParams(), where hiddenToVariationalParams() depends on the model.
  \item Decoder structure: Linear($L$, 8*8*32), Reshape((8,32,32)), Conv2DTranspose(out=64, k=3, strides=2, padding=same), ReLU(), Conv2DTranspose(out=32, k=3, strides=2, padding=same), ReLU(), Conv2DTranspose(out=1, k=3, strides=1, padding=same), Reshape((32, 32, 1)).
  \item If KVAE or VRNN, we had to applying KL and Adam optimizer step size schedulers in order to make them work, meaning that they are more intricately optimised than the other models.
  \item If GPVAE and MGPVAE, initialise the kernel lengthscales with 40 for each latent channel. The kernel hyperparameters (lengthscales and scales) are subsequently optimised via the ELBO.
  \item The model at the last training epoch is used for test evaluation.
\end{itemize}

\textbf{Remark:} We hereby explain why GPVAE \citep{fortuin_gp-vae_2020} is not practically suited for the missing frames imputation task. Recall that their approximate posterior is defined as $q(z_{1:T}|x_{1:T})=N(m_j, \Lambda_j^{-1})$, where $\Lambda_j=B_j^\intercal B_j$ and $B_j$ is an upper triangular band matrix, parameterised by outputs from the encoder. For a batch of sequences with varying lengths after removing missing frames, $B_j$ and hence $\Lambda_j$ will be padded with zeros (assuming we rearrange the orders of the time points). This is because most mainstream deep learning libraries (e.g. JAX, PyTorch and TensorFlow) all require static shape arrays/tensors and batch operations can only be done when all the arrays/tensors are of the same shape. However, we are not aware of any computationally efficient methods to “invert” a batch of $\Lambda_j$’s with varying 0-padding sizes. Whilst it is still feasible to implement GPVAE for missing data imputation tasks, either the batch size will have to be 1 or we will need an extra for loop to iterate through each of the sequences in each batch during training. This severely constrains the practical computational efficiency of GPVAE, rendering it unsuitable as a missing frames imputation model.

\textbf{Implementation Details:} VRNN and KVAE are implemented in PyTorch. GPVAE mostly a non-modified version as the original TensorFlow implementation in \citet{fortuin_gp-vae_2020}. MGPVAE is implemented in JAX using Objax \citep{objax2020github}. One issue in Objax is that the convolutional operator only supports float32 precision (otherwise the implementation will be very slow) and therefore we need to manually cast the CNN weights and inputs into float32 (from float64). This is a slight disadvantage of MGPVAE that hopefully can be resolved in the future through further advances in JAX-based CNN implementations.

During test time, we used $K=20$ latent samples to compute the NLL and RMSE (via the posterior mean). The NLL and RMSE of the predictive results on the test set are calculated over the entire sequence.

\begin{figure}[t]
\centering
\begin{minipage}{\textwidth}
\centering
\includegraphics[width = \columnwidth]{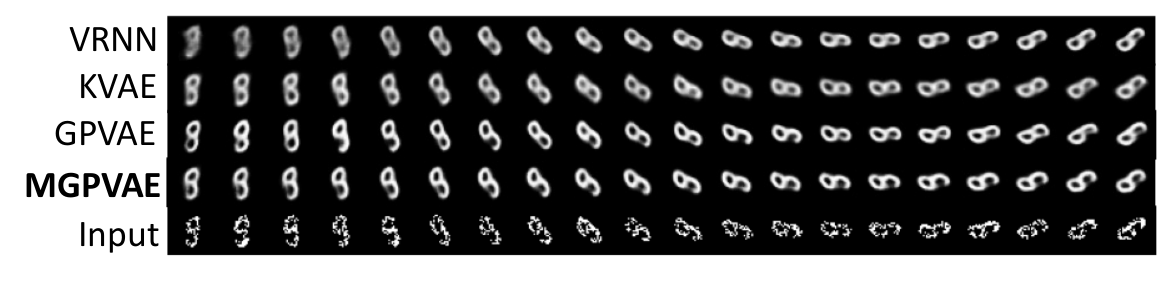}
\includegraphics[width = \columnwidth]{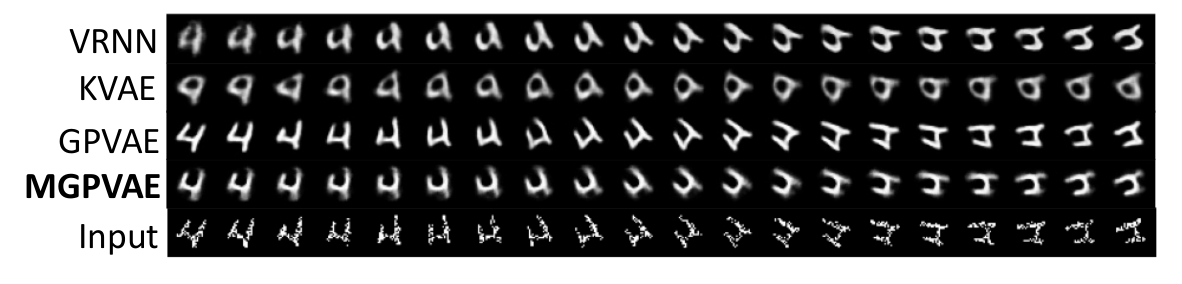}
\end{minipage}
\caption{Additional corrupt frames imputation results.}
\end{figure}

\begin{figure}[t]
\centering
\begin{minipage}{\textwidth}
\centering
\includegraphics[width = \columnwidth]{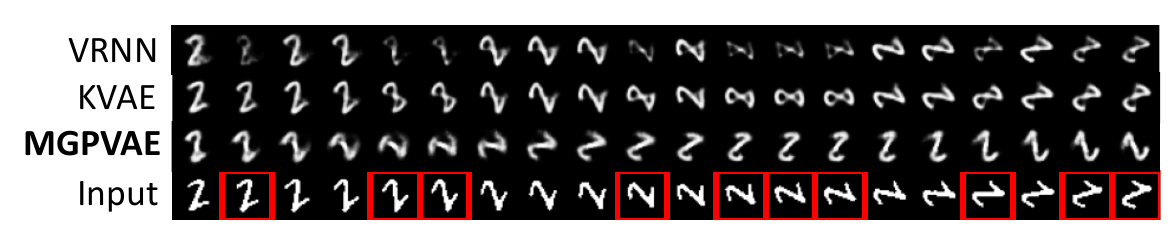}
\includegraphics[width = \columnwidth]{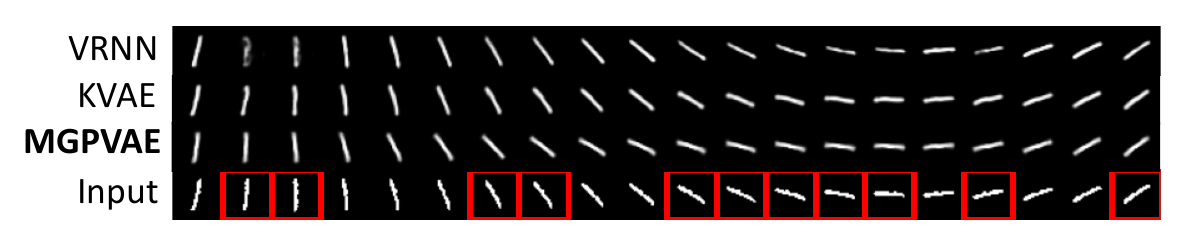}
\end{minipage}
\caption{Additional missing frames imputation results.}
\end{figure}

\subsection{Mujoco Action Data}
\label{appendix:action}
We create missing time steps by randomly dropping out $60\%$ of them.

For each model during training, we used the following configurations (note that we use the standard PyTorch-TensorFlow-Jax notation for network layers). Note that unlike \citet{rubanova2019latent}, we do not use a masked autoencoder approach for training and only use the non-missing time steps to compute the ELBO.
\begin{itemize}
  \item Batch size: 16
  \item Training epochs: 1000 if $T=100$ and 500 if $T=1000$.
  \item Number of latent channels: $L=15$
  \item Adam optimizer learning rate: 1e-3.
  \item clipGradNorm(model parameters, 100)
  \item Encoder structure: Linear($L$, 32)-ReLU-Linear(32, L). If latentODE then slightly different but similar due to the ODE-RNN.
  \item Decoder structure: Linear($L$, 16)-ReLU-Linear(16, 14).
  \item If KVAE or VRNN, we had to applying KL and Adam optimizer step size schedulers in order to make them work, meaning that they are more intricately optimised than the other models. LatentODE also used a KL scheduler.
  \item If SVGPVAE or MGPVAE, initialise the kernel lengthscales with: 5 if $T=100$ and 50 if $T=1000$. The kernel hyperparameters (lengthscales and scales) are subsequently optimised via the ELBO.
  \item SVGPVAE, VRNN, latentODE and MGPVAE we use early stopping with the validation RMSE on 320 if $T=100$ else 80 validation sequences (not part of train or test sets). For KVAE and CRU, adding cross-validation is practically too time consuming and so the model at the last training epoch is used for test evaluation.
\end{itemize}

\textbf{Remark:} CRU \citep{schirmer2022modeling} is capable of modelling continuous data in the form $(t_i, y_i)_{i=0}^N$, but, unlike latentODE, SVGPVAE or MGPVAE, it is unable to condition on $(t_i, y_i)_i$ and then predict at any time step $t\in (t_0, t_N)$. Therefore to train the CRU for our task, we have to treat it the same way as VRNN and KVAE, where we perform Kalman filtering in regular time steps and then mask out the unobserved time steps during training. Furthermore, this model is trained via maximum likelihood estimation, and so we only report the test RMSE metric.

\textbf{Implementation Details:} VRNN, KVAE, latentODE, SVGPVAE and CRU are implemented in PyTorch. For latentODE and CRU, they are mostly unmodified, based off the original author's implementations. We rewrite SVGPVAE using the more efficient framework of Functorch \citep{functorch2021}, which we find to give better implementation efficiency and performance than the original TensorFlow implementation. MGPVAE is implemented with Objax \citet{objax2020github} within JAX.

During test time, we use $K=20$ latent samples to compute the NLL and RMSE (via the posterior mean). The NLL and RMSE of the predictive results on the test set are calculated over the entire sequence.

\begin{figure}[tbp]
\centering
\begin{minipage}{\textwidth}
\centering
\includegraphics[width = \columnwidth]{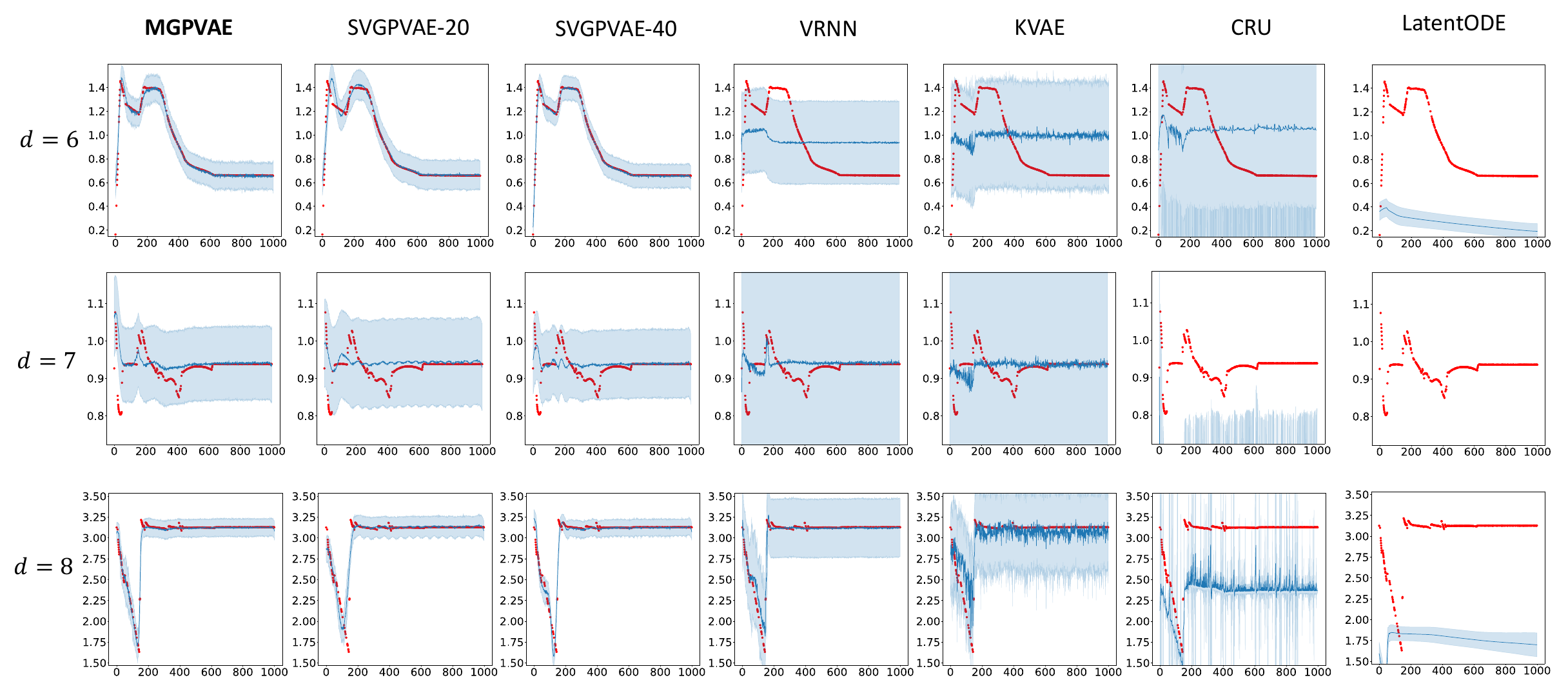}
\includegraphics[width = \columnwidth]{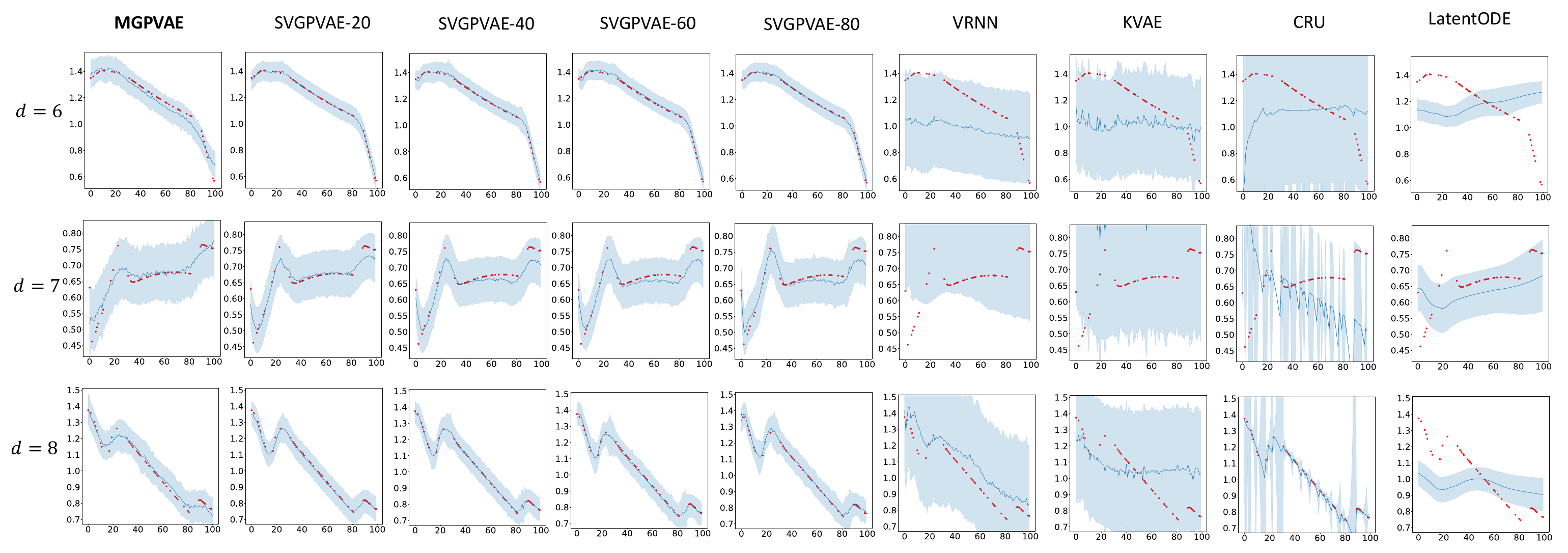}
\end{minipage}
\caption{Additional results figures for unseen Mujoco sequences. Note that some predictions are not showing as they fall outside the graph limits.}
\label{fig:additional_fig_mujoco}
\end{figure}
\newpage

\subsection{Spatiotemporal Data}
We download the ECMWF ERA5 data (ECMWF/ERA5/DAILY) from Google Earth Engine API \citep{gorelick2017google}. The data vectors contain: 
\begin{itemize}
  \item Mean 2m air temperature 
  \item Minimum 2m air temperature 
  \item Maximum 2m air temperature 
  \item Dewpoint 2m air temperature 
  \item Surface pressure 
  \item Mean sea level pressure 
  \item $u$ component of wind 10m
  \item $v$ component of wind 10m
\end{itemize}
The spatial locations are represented by [longitude, latitude, elevation]. We index time by $0,1,\ldots, 100$.

The model configurations for SVGPVAE and MGPVAE are:
\begin{itemize}
  \item clipGradNorm(model parameters, 100)
  \item Encoder structure: Linear($L$, 16)-ReLU-Linear(16, L). 
  \item Decoder structure: Linear($L$, 16)-ReLU-Linear(16, 14).
  \item The model at the last training epoch is used for test evaluation.
\end{itemize}

\textbf{Implementation Details:} MOGP is implemented using GPFlow \citep{matthews2017gpflow} within TensorFlow. SVGPVAE is again implemented in PyTorch with Functorch. MGPVAE is implemented with Objax within JAX.

\paragraph{MOGP and SVGPVAE:} We stack time and space together in the dataset, and have inducing points over space-time jointly. During training, we use only 1 latent sample to compute the ELBO for SVGPVAE, a minibatch size of 40 and 200 epochs over the entire dataset (this gives roughly 20,000 gradient steps). We use an Adam optimizer learning rate of 1e-3. Similary for MOGP, we train over 20000 epochs of minibatches of size 40.

\paragraph{MGPVAE:} We don't do any minibatching and directly use the entire training dataset, which is of dimension roughly 60-100-8. During training, we compute the ELBO with 20 latent samples and train over 20000 epochs. We use an Adam optimizer learning rate of 1e-3.

\paragraph{Prediction:} We compute the NLL with 20 latent samples over a test set of unseen spatial locations at all time steps.

\begin{figure}[tbp]
\centering
\begin{minipage}{\textwidth}
\centering
\includegraphics[width = \columnwidth]{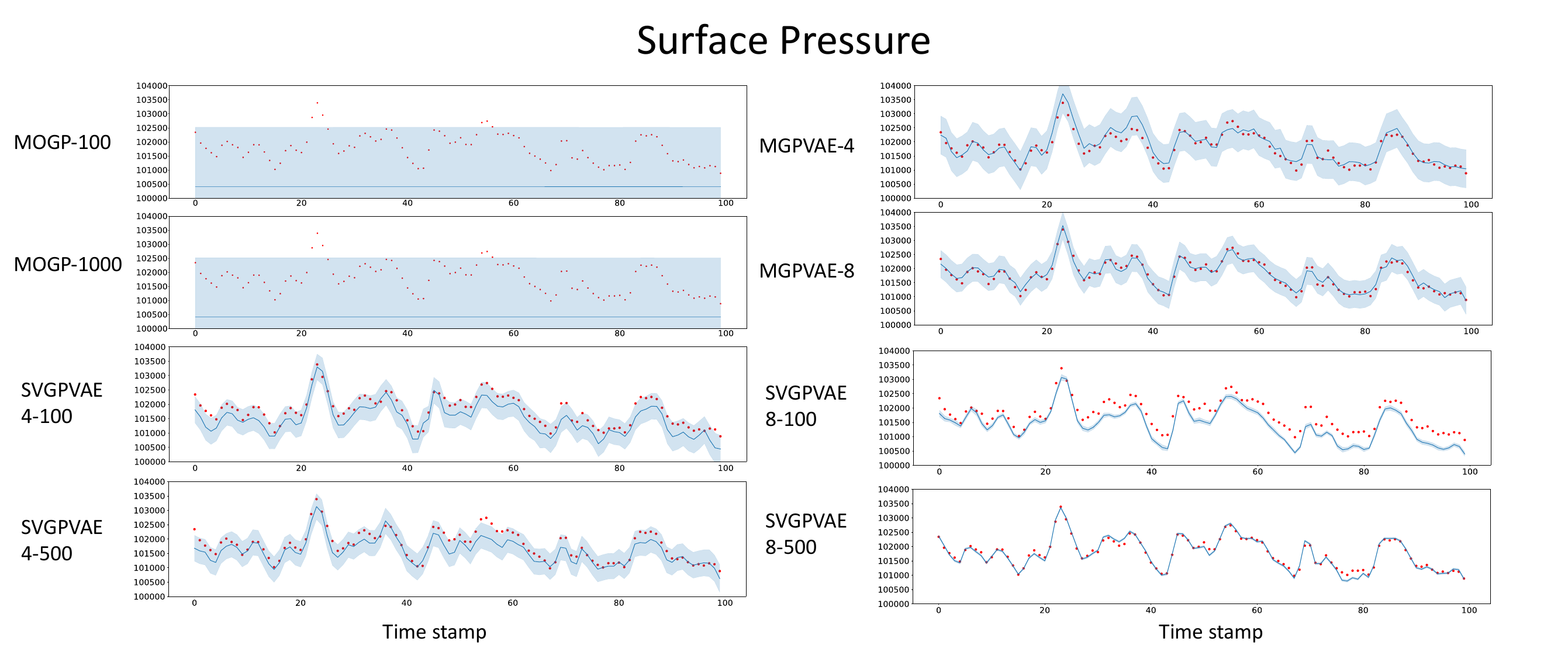}
\includegraphics[width = \columnwidth]{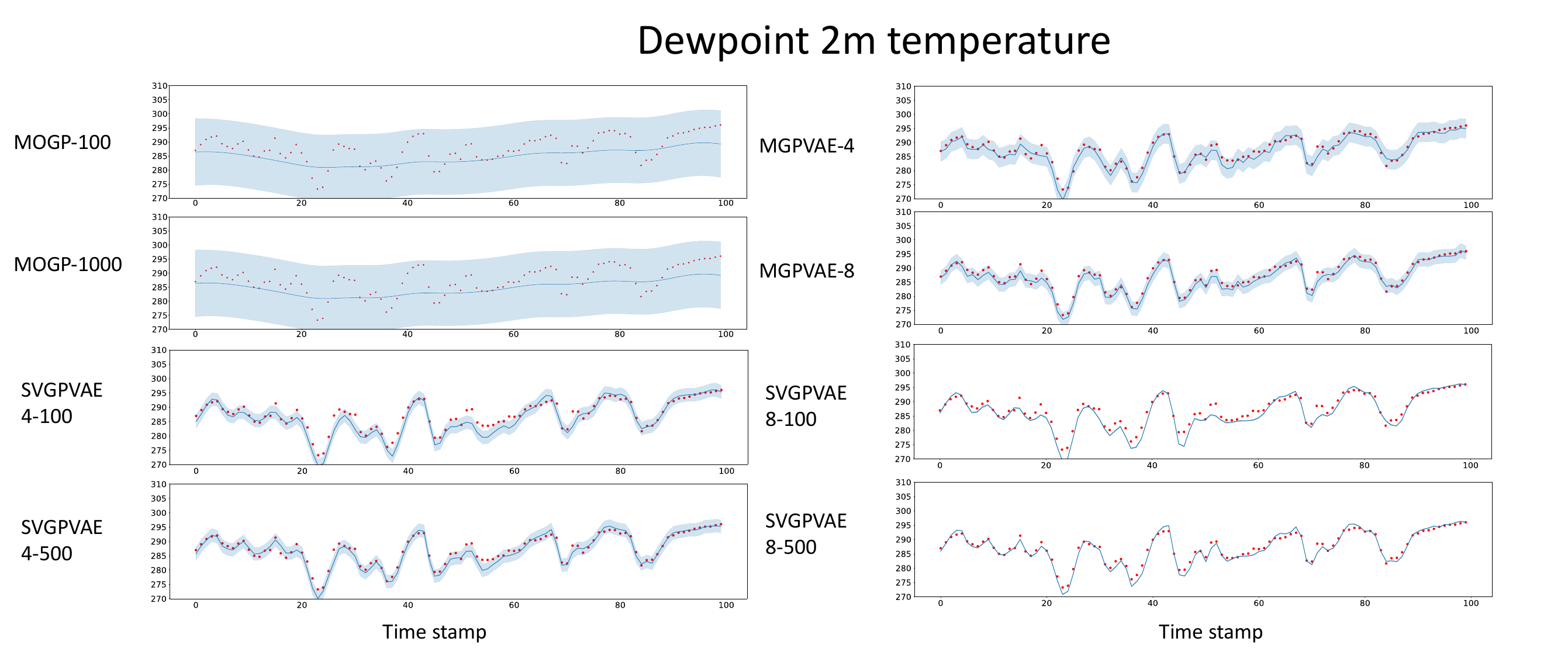}
\end{minipage}
\caption{Additional figures of the results for the ERA5 experiment given fixed spatial locations.}
\end{figure}

\section{Possible Future Developments}
It is possible to extend MGPVAE to incorporate temporal and spatial inducing points, by the virtue of the works of \citet{adam_doubly_2020} and \citet{hamelijnck_spatio-temporal_2021}. The main challenge would be forming the likelihood approximation, which would have to rely on set encoders such as DeepSets \citep{zaheer2017deep} or Set Transformer \citep{lee2019set} to output approximate likelihoods factorised over inducing points (instead of over time points). For temporal data, the main challenge would be to cluster the time points together for each inducing time location. For spatiotemporal data, the clustering problem extends to space-time clusters.

The state space may also be enriched via normalising flows \citep{lin2023towards, maronas2022efficient}, which may additionally enrich the latent dynamics or simplify the learning dynamics. Another approach to enforcing global properties without GPs would be via a GAM-type model, such as that of Prophet \citep{taylor2018forecasting}.

\end{document}